\documentclass[final]{colt2025} %
\title[Discrete Distribution Estimation in KL Divergence]{Nearly Minimax Discrete
Distribution Estimation\\in Kullback-Leibler Divergence with High Probability}
\usepackage{times}

\usepackage{amsfonts}
\usepackage{bm}
\usepackage{bbm}

\usepackage{color} 
\usepackage{float}
\usepackage{algorithm}
\usepackage{algpseudocode}
\usepackage{enumitem}
\usepackage{comment}
\usepackage{framed}
\usepackage{dirtytalk}

\usepackage{thmtools, thm-restate}

\newcommand{\pest}{\hat p}
\newcommand{\pie}{p^*(i)}
\newcommand{\p}{p^*}

\newcommand{\phat}{\hat{p}}

\newcommand{\pp}{p^{+}}
\newcommand{\ppi}{p^{+}(i)}
\newcommand{\pbti}{\bar{p}_t(i)}
\newcommand{\pti}{p_t(i)}
\newcommand{\ptpi}{p_{t+1}(i)}
\newcommand{\pt}{p_t}
\newcommand{\ptp}{p_{t+1}}

\newcommand{\pnhpi}{p_{n/2 + 1}(i)}
\newcommand{\pnhp}{p_{n/2 + 1}}

\newcommand{\nit}{n_{i}(t)}

\newcommand{\nitm}{n_{i}(t-1)}
\newcommand{\nin}{n_{i}}
\newcommand{\ninh}{m_i}

\newcommand{\potb}{p^{\textnormal{OTB}}}
\newcommand{\potbi}{p^{\textnormal{OTB}}(i)}
\newcommand{\Jt}{\tilde{J}}

\newcommand{\sumK}{\sum_{i=1}^K}

\newcommand{\sumnhp}{\sum_{t = n/2 + 1}^n}
\newcommand{\sumt}{\sum_{s=1}^t}
\newcommand{\sumT}{\sum_{t=1}^T}

\newcommand{\pml}{\bar p_n}
\newcommand{\normaldist}{\mathcal{N}}
\newcommand{\weaklyconvergesto}{\rightsquigarrow}

\newcommand{\der}{\mathrm{d}}

\newcommand{\loweps}{\underline{\epsilon}}   

\renewcommand{\hat}{\widehat}
\renewcommand{\tilde}{\widetilde}
\newcommand{\regret}{\mathcal{R}}

\newcommand{\Jcal}{\mathcal{J}}
\newcommand{\Jtcal}{\tilde{\mathcal{J}}}

\newcommand{\model}{\triangle^K}
\newcommand{\domain}{\mathcal{X}}
\newcommand{\reals}{\mathbb{R}}

\newcommand{\half}{\tfrac{1}{2}}
\newcommand{\id}{\mathbbm{1}}
\newcommand{\indicator}{\id}

\newcommand{\pa}[1]{\left(#1\right)}

\DeclareMathOperator{\KL}{KL}
\DeclareMathOperator{\dist}{d}
\DeclareMathOperator{\E}{\mathbb E}
\DeclareMathOperator{\Pp}{\mathbb P}
\renewcommand{\Pr}{\Pp}
\DeclareMathOperator{\Vol}{Vol}         

\DeclareMathOperator*{\argmin}{arg\,min}
\DeclareMathOperator*{\argmax}{arg\,max}

\newcommand{\Dirk}[1]{%
\ifmmode
\text{\textcolor{orange}{Dirk: #1}}
\else
\textcolor{orange}{Dirk: #1}
\fi
}

\newcommand{\Julia}[1]{%
\ifmmode
\text{\textcolor{violet}{Julia: #1}}
\else
\textcolor{violet}{Julia: #1}
\fi
}

\newcommand{\Tim}[1]{%
\ifmmode
\text{\textcolor{magenta}{Tim: #1}}
\else
\textcolor{magenta}{Tim: #1}
\fi
}

\newcommand{\TODO}[1]{%
\ifmmode
\text{\textcolor{red}{TODO: #1}}
\else
\textcolor{red}{TODO: #1}
\fi
}

\renewcommand{\log}{\ln}

\coltauthor{%
  \Name{Dirk {van der Hoeven}} \Email{dirk@dirkvanderhoeven.com}\\
 \addr Leiden University, the Netherlands
 \AND
 \Name{Julia Olkhovskaia} \Email{julia.olkhovskaya@gmail.com}\\
 \addr TU Delft, the Netherlands%
 \AND
 \Name{Tim {van Erven}} \Email{tim@timvanerven.nl}\\
 \addr University of Amsterdam, the Netherlands%
}

\begin{document}

\maketitle

\begin{abstract}%
We consider the fundamental problem of estimating a discrete
distribution on a domain of size~$K$ with high probability in
Kullback-Leibler divergence. We provide upper and lower bounds on the
minimax estimation rate, which show that the optimal rate is between $\big(K +
\ln(K)\ln(1/\delta)\big) /n$ and $\big(K\ln\ln(K) +
\ln(K)\ln(1/\delta)\big) /n$ at error probability $\delta$ and sample
size $n$, which pins down the rate up to the doubly logarithmic factor
$\ln \ln K$ that multiplies $K$. Our upper bound uses techniques from
online learning to construct a novel estimator via online-to-batch
conversion. Perhaps surprisingly, the tail behavior of the minimax rate
is worse than for the squared total variation and squared Hellinger
distance, for which it is $\big(K + \ln(1/\delta)\big) /n$, i.e.\
without the $\ln K$ multiplying $\ln (1/\delta)$. As a consequence,
we cannot obtain a fully tight lower bound from the usual reduction to
these smaller distances. Moreover, we show that this lower bound cannot
be achieved by the standard lower bound approach based on a reduction to
hypothesis testing, and instead we need to introduce a new reduction to what we call
weak hypothesis testing.
We investigate the source of the gap with other divergences further
in refined results, which show that the total variation rate is
achievable for Kullback-Leibler divergence after all (in fact by the
maximum likelihood estimator) if we rule out outcome probabilities
smaller than $O(\ln(K/\delta) / n)$, which is a vanishing set as $n$
increases for fixed $K$ and~$\delta$. This explains why minimax
Kullback-Leibler estimation is more difficult than asymptotic
estimation.
\end{abstract}


\section{Introduction}\label{sec:introduction}

Consider a sample $S = (X_1,\ldots,X_n)$ of $n$ independent, identically
distributed observations of a random variable $X$ with a finite number
of possible values $[K] := \{1, \ldots, K\}$. The aim of discrete
distribution estimation
\citep{devroye2001combinatorial,diakonikolas2016learning,canonne2020short,polyanskiy2025}
is to approximate the unknown true probability mass function $\p$ of
$X$ by an estimate $\hat p$ based on the sample $S$, in terms of a given
measure of distance or divergence $\dist(\p,\hat p)$. In this work we
consider the case where $\dist$ is the Kullback-Leibler (KL) divergence
of $\p$ from~$\hat p$:
\[
  \KL(\p \| \hat p) = \sum_{i=1}^K \p(i) \ln \frac{\p(i)}{\hat p(i)},
\]
with the understanding that $0 \ln 0/a = 0$ for any $a \geq 0$ and $b \ln (b/0)
= \infty$ for $b > 0$. Specifically, we are interested in the minimax
rate with high probability: for any given $\delta \in (0,1)$, what is
the smallest bound $r_n^*(\delta)$ on the KL divergence that can the
guaranteed by any estimator $\hat p$ uniformly over all possible true
distributions $\p$ with confidence at least $1-\delta$, as $n$ grows
large? This question is equivalent to determining the sample complexity
in PAC-learning with the log loss $\ell(x,p) = -\ln p(x)$, for which the
Kullback-Leibler divergence equals the excess risk:
\[
  \KL(\p\|\hat p) = \E[\ell(X,\hat p)] - \E[\ell(X,\p)].
\]
This is of special interest, because the log loss is unbounded, so it
provides perhaps the simplest setting in which to study unbounded
losses.

At first sight, it may seem natural to conjecture that the minimax rate
might be of order $\big(K + \ln(1/\delta)\big)/n$. Although this will
turn out to be false, we start by reviewing the many known results that
point in this direction.

One clue comes from the properties of the maximum likelihood
estimator (MLE) $\pml(i) = n_i/n$, where $n_i$ is the number of times
that outcome $i$ occurs in $S$. The MLE achieves the minimax rate
of order $\sqrt{\frac{K + \ln(1/\delta)}{n}}$ for, amongst others, the
total variation distance $V(\p,\hat p) = \sum_{i=1}^K |\p(i) - \hat
p(i)|$ and the Hellinger distance $H(\p,\hat p) = \sqrt{\sum_{i=1}^K
\big(\sqrt{\p(i)} - \sqrt{\hat p(i)}\big)}$
\citep{canonne2020short}.\footnote{NB There exist multiple conventions
in the literature for the scale factors in the definitions of $V$ and
$H$.} Since
\begin{equation}\label{eqn:variation_Hellinger_KL}
  \tfrac{1}{4} V(p,q)^2 \leq H(p,q)^2 \leq \KL(p\|q)
  \qquad
  \text{for any $p,q$}
\end{equation}
\citep{Tsybakov2009}, it follows that the minimax rate for the KL
divergence is at least of order $\frac{K + \ln(1/\delta)}{n}$.
Furthermore, \citet{agrawal2022finite} shows that this is also the
minimax rate for KL with reversed arguments, $\KL(\hat p\|\p)$,
achieved again by the MLE. Finally, as pointed out by
\citet{mourtada2025estimation}, if we fix $\p$ and let $n$ tend to
infinity, then the MLE is asymptotically normally distributed around
$\p$, and consequently a second-order Taylor approximation of the KL
divergence implies (by the second-order delta method,
\citealp{VanDerVaart1998}) the following convergence in distribution:
\begin{equation}\label{eqn:KLasymptotics}
  n \KL(\p \| \pml) ~\weaklyconvergesto~ \tfrac{1}{2} Y,
  \qquad
  \text{where $Y \sim \chi_{M-1}^2$.}
\end{equation}
Here, $\chi_{M-1}^2$ denotes a $\chi^2$ distribution with $M-1$ degrees
of freedom, and $M \leq K$ is the size of the support of~$\p$. See
Theorem~\ref{thm:MLEWeakConvergence} in Appendix~\ref{app:MLEasymp} for
details. By concentration of $Y$ around its mean
\citep{LaurentMassart2000}, we have that $Y \leq C(M - 1 +
\ln(1/\delta))$ with probability at least $1-\delta$ for some constant
$C>0$, which leads to
\begin{equation}\label{eqn:Gaussiantailsrate}
  \KL(\p \| \pml) \leq C\Big(\frac{M + \ln(1/\delta)}{n}\Big)
  \quad
  \text{for all sufficiently large $n$, \qquad if $\p$ and $\delta$ are
  fixed.}
\end{equation}
Unfortunately, it is readily seen that the MLE fails to achieve the
minimax rate for our setting of interest: suppose there exists an
outcome~$i$ that has positive probability $\p(i) > 0$ but has not been
observed in the sample: $n_i = 0$. Then $\p(i)/\pml(i) = \infty$ and
consequently $\KL(\p\|\pml) = \infty$. But, for any $n$, it is possible
to take $\p(i)$ small enough that this happens with probability larger
than~$\delta$, so the MLE cannot guarantee any finite minimax rate.

A common solution is to smooth the MLE by adding a fake number of
observations $\gamma > 0$ to each possible outcome. The resulting
add-$\gamma$ estimator is $p^\gamma(i) = \frac{n_i + \gamma}{n +
K\gamma}$. For $\gamma=1$, this is also called the Laplace estimator,
and for $\gamma=1/2$ it is known as the Krichevsky-Trofimov (KT)
estimator \citep{krichevsky1981performance}. Following a sequence of
prior results
\citep{Cover1972,krichevsky1981performance,BraessForsterSauerSimon2002,Paninski2004},
\citet{braess2004bernstein} show that a variant of the add-$\gamma$
estimator with data-dependent smoothing achieves the minimax rate
\emph{in expectation}:
\[
  \min_{\hat p} \max_{\p} ~\E_{\p}[\KL(\p\|\hat p)]
    ~=~ \frac{K-1}{2n} +
  o\Big(\frac{1}{n}\Big),
\]
and \citet{catoni1997mixture} shows that the
Laplace estimator achieves this rate up to a factor of $2$:
\begin{equation}\label{eqn:LaplaceExpectationBound}
  \E_{\p}[\KL(\p\|p^1)] \leq \frac{K-1}{n}
  \qquad \text{for all $\p$.}
\end{equation}
See also the proofs by \citet{mourtada2019improper} and, for $K = 2$, by
\citet{forster2002relative}. \citet{KamathOrlitskyPichapatiSuresh2015}
further characterize the minimax rate in expectation for, amongst
others, the $\ell_2^2$, $\ell_1$ and $\chi^2$ divergences.

Although suggestive, since all results mentioned so far are for
different settings, the only formal implication for our setting is that
$\big(K + \ln(1/\delta)\big)/n$ is a lower bound on the minimax rate.
This is not matched by any of the following known upper bounds: First, a
sequence of increasingly tight bounds for the Laplace estimator have
been improved from order $K \ln(n) \ln(K/\delta)/n$
\citep{bhattacharyya2021near} to $\frac{K + \sqrt{K}
\ln(1/\delta)^3}{n}$ \citep{han2021optimal} to the tightest known
guarantee by \citet{canonne2023concentration}, who show that
\begin{equation}\label{eqn:LaplaceRate}
    \KL(\p\|p^1)
      \leq \E[\KL(\p\|p^1)] + \frac{C\sqrt{K\ln(K/\delta)^{5}}}{n}
      \leq \frac{K-1}{n} + \frac{C\sqrt{K\ln(K/\delta)^{5}}}{n}
\end{equation}
with probability at least $1-\delta$, for some absolute constant $C >
0$, where the second inequality holds deterministically by
\eqref{eqn:LaplaceExpectationBound}. The remaining gap with the lower
bound is seen both in the exponent $5/2$ on $\ln(1/\delta)$ and in the
factor $\sqrt{K}$ that multiplies it. \citet{canonne2023concentration}
further show that very little further improvement is possible via
concentration of $\KL(\p\|p^1)$ around its mean, because the deviation
in \eqref{eqn:LaplaceRate} cannot be improved to
$\frac{CK^\eta\ln(1/\delta)}{n}$ for any $\eta < 1/2$, at least not
uniformly for all $\delta$.

An alternative estimator is proposed by \citet{vanderHoeven2023high},
who obtain a novel type of regret guarantee for their online learning
algorithm, and then convert the algorithm to an estimator using
online-to-batch (OTB) conversion \citep{Littlestone1989}. This estimator
guarantees that
\begin{align*}
    \KL(\p\|\hat p) \leq C\frac{K + \log(n)\log(1/\delta)}{n}
\end{align*}
with probability at least $1-\delta$ for any~$\p$, where $C > 0$ is an
absolute constant. This improves the exponent on $\ln(1/\delta)$
compared to \eqref{eqn:LaplaceRate}, but leaves a multiplicative factor
$\log(n)$ that is not present in the lower bound.

\paragraph{Main Results}

We provide improved upper and lower bounds for discrete distribution
estimation in KL divergence, which show that the minimax rate is
sandwiched between 
\begin{equation}\label{eqn:minimaxLowerUpper}
  C_1 \Big(\frac{K + \ln(K)\ln(1/\delta)}{n}\Big)
  ~\leq~
  r_n^*(\delta)
  ~\leq~
  C_2 \Big(\frac{K\ln\ln(K) + \ln(K)\ln(1/\delta)}{n}\Big)
\end{equation}
for absolute constants $0 < C_1 \leq C_2$. This characterizes the exact
minimax rate up to a doubly logarithmic factor in $K$. 

Our upper bound is obtained by converting an online learning algorithm
to an estimator $\potb$ using online-to-batch conversion with
suffix-averaging, i.e.\ averaging only the predictions on the second
half of the data \citep{rakhlin2011making, harvey2019tight,
aden2023optimal}. In fact, we obtain a potentially stronger bound: there
exists $C > 0$ such that
\begin{equation}\label{eqn:OTBBound}
  \KL(\p\|\potb)
    \leq C \frac{K + \Jt\ln(\ln(\Jt)) + \ln(J)\ln(1/\delta)}{n}
\end{equation}
with probability at least $1-\delta$, where $J \leq K$ is the number of
outcomes $i$ occurring fewer than $32\ln(K/\delta)$ times in the first
half of the data $X_1,\ldots,X_{n/2}$ and $\Jt \leq K$ is the number of
$i$ for which $\p(i) < 32 \ln (K/\delta)/n$, except that both $J$ and
$\Jt$ are taken to be at least $3$.

In addition, we also analyze an alternative estimator $\tilde p$, which
is a variant of the add-$\gamma$ estimator in which $\gamma$ is
estimated on a hold-out set and is allowed to differ between outcomes $i$.
This estimator satisfies
\begin{equation}\label{eqn:AddGammaBound}
    \KL(\p \| \tilde{p})
      \leq C\Big(1 + \ln\big(\min\{\ln(K/\delta), J\}\big)\Big) \frac{K +
      \ln(1/\delta)}{n},
\end{equation}
with probability at least $1-\delta$, for some $C>0$. Although this
gives a slightly worse bound than~\eqref{eqn:OTBBound} for the worst
case that $\Jt = J = K$ (see Remark~\ref{rem:rate_comparison}),  it is
still within a doubly logarithmic factor of the lower bound, and the
estimator is simpler.

Finally, as a side-result, we return to the MLE: as discussed
above, it works asymptotically, but may fail for any finite~$n$ if there
exist outcomes $i$ with $\p(i) > 0$ too small compared to $n$. We
provide a threshold for when probabilities are `too small' by showing
that, with probability at least $1-\delta$,
\begin{equation}\label{eqn:MLEFiniteNBound}
  \KL(\p \| \bar{p}_{n}) \leq C \frac{K + \log(1/\delta)}{n}
  \qquad
  \text{for all $\p$ such that $\p(i) \geq \frac{32 \ln (K/\delta)}{n}$ for
  all $i$,}
\end{equation}
for some absolute constant $C > 0$. Since this is again the asymptotic
rate, which beats the lower bound in \eqref{eqn:minimaxLowerUpper}, we
also see that these small probabilities cause a difference in the rate.
The result in~\eqref{eqn:MLEFiniteNBound} will actually be proved via a
bound on the $\chi^2$ divergence, which can be arbitrarily larger than
the KL divergence in general, but is within a constant factor with high
probability in this case.

Regarding the lower bound: when $K/n$ dominates $\ln(K)\ln(1/\delta)/n$,
the lower bound follows directly from the known lower bound for squared
variational distance \citep{canonne2020short} and
\eqref{eqn:variation_Hellinger_KL}. (We still include a proof for lack
of a reference that explicitly spells out the details.) The interesting
part of the lower bound is the case when $\ln(K)\ln(1/\delta)/n$ is
dominant. We first obtain this rate from an estimator-dependent lower
bound, which makes explicit the trade-off when smoothing probability
estimates for outcomes~$i$ that have not been observed, i.e.\ for which
$n_i=0$. This lower bound further implies that the only add-$\gamma$
estimators that can achieve a rate of $\ln(K)\ln(1/\delta)/n$ must have
$\gamma \propto \frac{\ln(1/\delta)}{K}$, whereas any constant
$\gamma$ that does not depend on $\delta$, like in the Laplace and KT
estimators, will lead to a rate with order $\ln(1/\delta)\ln
\ln(1/\delta)/n$ dependence on $\delta$, which is worse than
$\ln(K)\ln(1/\delta)$ in the regime where $\ln(K)\ln(1/\delta)$
dominates $K$. Although the estimator-dependent bound is informative, it
is proved using an ad hoc argument, so we further investigate whether
the same lower bound can also be obtained from a general technique.
Surprisingly, we find that neither Fano's inequality nor any other
approach based on the standard reduction to hypothesis testing
\citep{Tsybakov2009} can work. Instead, in order to recover the correct
rate, we need to introduce a new type of reduction to what we call a
\emph{weak hypothesis testing} problem, where the goal is not to
identify the true distribution from a finite set of candidates, but
merely to identify any incorrect distribution.

\paragraph{Concurrent Work}

In very insightful independent and concurrent work,
\citet{mourtada2025estimation} derives matching lower and upper bounds
(up to constants) for the same setting we consider.%
\footnote{\emph{Note on concurrency:} The results in this paper were developed independently and contemporaneously with those of Mourtada. An earlier version of our work containing these results, available as \citep{MultinomKLv1}, was submitted to COLT in February 2025, prior to the appearance of Mourtada’s April 2025 arXiv preprint, thereby establishing a verifiable submission record. We subsequently improved the presentation and submitted the present version to ALT in October 2025.}
For estimators that
are allowed to depend on $\delta$, he shows that the lower bound in
\eqref{eqn:minimaxLowerUpper} is actually tight, and that it is in fact
achieved by the add-$\gamma$ estimator with $\gamma = 1 \vee
\frac{\ln(1/\delta)}{K}$. This is possible by his crucial observation
that the standard Chernoff method based on the moment generating
function cannot provide tail bounds of the right form, so instead he
proceeds by controlling moments directly. His lower bound uses the same
construction as our estimator-dependent lower bound.
\citeauthor{mourtada2025estimation} further considers estimators that do not
depend on $\delta$, where he shows that the Laplace estimator achieves:
\[
    \KL(\p\|p^1)
      \leq C\frac{K + \ln(1/\delta)\ln \ln(1/\delta)}{n}
\]
with probability at least $1-4\delta$, for some $C > 0$, which matches
the rate in our estimator-dependent lower bound. He also shows that this
is the minimax rate among the restricted, but large class of all
estimators $\hat p$ that a) do not depend on $\delta$ and b) achieve
$\KL(\p\|\hat p) \leq \kappa K/n$ with positive probability for all
$\p$, where $\kappa \geq 1$ can be any absolute constant. Finally,
\citeauthor{mourtada2025estimation} also provides results for $\p$ with
sparse support, which are of interest in the regime where $K$ is large
compared to $n$.

\paragraph{Further Related Work}

A related goal is to minimize the minimax regret in an online learning
setting with logarithmic loss, which is known to be
$\frac{K-1}{2}\log(n) + O(1)$ and is achieved by the KT estimator.
For $K = 2$, this was shown by \citet{krichevsky1981performance}, and
\citet{xie1997minimax} generalize it to general $K$. For a more
comprehensive overview we refer to \citet{shtar1987universal,merhav98,
rissanen1996fisher}. Our bounds for $\potb$ do not directly arise from
combining these results with online-to-batch conversion, but instead
rely on a new analysis of suffix averaging to avoid the $\log(n)$ factor
that appears in the minimax regret. 

\paragraph{Outline}

The remainder of the paper is organised as follows. In Section~\ref{sec:upperbounds} we provide detailed statements for our upper bounds as well as the proofs for our main results. In Section~\ref{sec:lowerbounds} we present detailed statements of our lower bounds. Finally, we provide a brief discussion and outlook in Section~\ref{sec:conclusion}.

\paragraph{Notation} 
For $\delta \in
(0,1]$, the worst-case rate of an estimator $\pest$ at confidence level
$1-\delta$ is %
\begin{equation}\label{eqn:worst_case_rate}
  r_n(\pest,\delta) = \inf \Big\{r_n \mid \sup_{\p \in \model}
  \Pr_{\p}\Big(\KL(\p\|\pest) > r_n \Big) \leq \delta\Big\}\,,
\end{equation}
where $\model = \{p: p(i) \geq 0, \sumK p(i) = 1\}$.
Then $r_n^*(\delta) = \inf_{\pest} r_n(\pest,\delta)$
is the minimax rate. 
We further let $B(n, p)$ denote the binomial
distribution with $n$ trials and success probability $p$, and we write
$M(n, p(1), \ldots, p(K))$ for the multinomial distribution with $n$
trials and probabilities $p(1), \ldots, p(K)$.

\section{Upper Bounds}\label{sec:upperbounds}

Here we provide detailed statements of our upper bounds. We start with
definitions we use throughout this section.  For simplicity, we assume
that $n/2$ is an integer throughout the section. We will consider a
version of the add-$\gamma$ estimator based on the first $t$ data points
and where $\gamma_i \geq 0$ is allowed to differ between outcomes $i$:
let $\ptpi =
\frac{\nit + \gamma_i}{t + \sumK\gamma_i }$, where $\nit = \sumt
\id[X_s = i]$. Note that, in particular, $p_{n+1}$ is the estimator
obtained when using the whole sample.
Let $\nin = n_i(n)$, $\ninh = n_i(n/2)$, $\pbti = \nit/t$, and $\ppi =
\frac{\pie n + \gamma_i}{n + \sumK \gamma_i}$. Let $\Jcal = \{i:\ninh < 32 \log(K/\delta)\}$, $J = \max\{3, |\Jcal|\}$, $\Jtcal = \{i: n\pie < 32\log(K/\delta)\}$, and $\Jt = \max\{3, |\Jtcal|\}$.  

Our first upper bound is for the maximum likelihood estimator: 
\begin{restatable}{retheorem}{thdivergences}\label{th:divergences}
    For some $C > 0$, if $|\Jtcal| = 0$, then, with probability at least
    $1-\delta$, the maximum likelihood estimator guarantees%
    \begin{align*}
         \tfrac{1}{6} \chi^2 (\p,\bar{p}_n) \leq \tfrac{1}{4} \chi^2 (\bar{p}_n,\p)  \leq \KL(\p \| \bar{p}_{n}) \leq C \frac{K + \log(1/\delta)}{n}\,.
    \end{align*}
\end{restatable}
The proof of Theorem~\ref{th:divergences} can be found in
Appendix~\ref{app:upperbounds}. Theorem~\ref{th:divergences} tells us
that if $\p$ is in the interior of the simplex (i.e.\ $|\Jtcal| = 0$), then the empirical mean is a good estimator in the sense that $\tfrac{1}{6} \chi^2 (\p,\bar{p}_n), \chi^2 (\bar{p}_n,\p)$, and $ \KL(\p \| \bar{p}_{n})$ %
are at most of order $(K + \ln(1/\delta))/n$ with probability at least $1 - \delta$.  %
However, the condition $|\Jtcal| = 0$ is not something we observe, which
implies that we cannot tell if $\bar{p}_n$ is a good estimator. Furthermore, Theorem~\ref{th:divergences} does not provide insight into what estimator one should use if $|\Jtcal| > 0$. In fact, if $|\Jtcal| > 0$, it can be seen that the maximum likelihood estimator cannot guarantee that $\KL(\p \| \bar{p}_{n})$ is finite with probability at least $1 - \delta$ (Section~\ref{sec:LBinstance}). 

In the remainder of this section we will provide upper bounds that apply more broadly. We will make use of the add-$\gamma$ estimator and $\potb = \frac2n \sumnhp \pt$, both with 
\begin{align}\label{eq:defgammai}
    \gamma_i & = \begin{cases}
        0 & \textnormal{ ~ if ~} \ninh \geq 32 \ln(4K/\delta) \\
        \max\{1, \frac{\log(K/\delta)}{J}\} & \textnormal{ ~ otherwise.} 
    \end{cases}
\end{align}
Note that we set $\gamma_i = 0$ if $\ninh$ is sufficiently large. This
allows us to avoid unnecessary bias when possible.
We first provide an upper bound for the add-$\gamma$ estimator: 
\begin{restatable}{retheorem}{thmeanrocks}\label{th:meanrocks}
    Suppose that $n \geq K$ and $n \geq \log(K/\delta)$. Then, with
    probability at least $1-2\delta$, the add-$\gamma$ estimator with
    $\gamma_i$ set according to \eqref{eq:defgammai} guarantees that
    \begin{align*}
        \KL(\p\|p_{n+1}) \leq \begin{cases}
            (2 + \log(40\min\{J, \log(K/\delta)\})) \Big(\frac{9K + 9\ln(1/\delta)}{n} \Big) & \textnormal{ ~ if ~ } |\Jcal| \geq 1 \\
            (2 + \log(2)) \Big(\frac{7K + 6\ln(1/\delta)}{n}\Big) & \textnormal{ ~ if ~ } |\Jcal| = 0\,.
        \end{cases}
    \end{align*}
\end{restatable}
Our lower bounds in Section~\ref{sec:lowerbounds} suggest that the $\log(\min\{J, \log(K)\})K$ term is
suboptimal in the worst case. Indeed, in Theorem~\ref{th:OTB} we show
that the following estimator provides a stronger guarantee:
\begin{equation}\label{eqn:def_otb}
  \potb = \frac2n \sumnhp \pt,
\end{equation}
where $\pt$ uses $\gamma_i$ from \eqref{eq:defgammai}. This estimator
is an instance of online-to-batch (OTB) conversion in which we average
only over the second half of the data, which is called suffix averaging.
As mentioned in Section~\ref{sec:introduction}, standard OTB conversions
that average over all $t=1,\ldots,n$, would introduce unnecessary
$\log(n)$ terms, which is why we resort to suffix averaging. For $\potb$
we obtain the following guarantee:
\begin{restatable}{retheorem}{thOTB}\label{th:OTB}
    Suppose that $\log(K/\delta) \leq n$ and that $K \leq n$. Then, with
    probability at least $1 - 5\delta$, the OTB estimator defined in
    \eqref{eqn:def_otb} guarantees that
    \begin{align*}
        & \KL(\p\|\potb) 
         \leq \frac{1}{n}\Big(128 K + 200 \log(800J)\log(1/\delta) + 8\Jt\log\left(24\left(\log\left(\frac{\Jt}{\log(1/\delta)}\right) \vee 1 \right)\right)\Big)\,.
    \end{align*}
\end{restatable}
Ignoring constants, the gap between our lower bounds in
Section~\ref{sec:lowerbounds} and the guarantee of $\potb$ is an
additive $\Jt\log\left(\log\left(\frac{\Jt}{\log(1/\delta)}\right) \vee
1\right)$. In the worst case, this term is of order $K\log(\log(K))$. %

\begin{remark}\label{rem:rate_comparison}
To see that Theorem~\ref{th:OTB} strictly improves on
Theorem~\ref{th:meanrocks} in the worst case that $\Jt = J = K$,
consider the following case distinction: If $K \leq \ln(1/\delta)$, then
both theorems give a rate of order $\ln(K)\ln(1/\delta)$ and the rates
are the same. If $K \geq \ln(1/\delta)$ then the rate for the
add-$\gamma$ estimator in Theorem~\ref{th:OTB} is of order $K \ln \ln K
+ \ln(K)\ln(1/\delta)$ and the rate in Theorem~\ref{th:meanrocks} is of
order $K\ln\big(\ln(K) + \ln(1/\delta)\big)$. We see that the latter
always exceeds $K\ln\ln K$, and is also at least of order
$K\ln\ln(1/\delta) \gtrsim \ln(K) \ln(1/\delta)$.  
\end{remark}
Both Theorem~\ref{th:meanrocks} and \ref{th:OTB} provide a stronger guarantee as $J = \max\{3, |\Jcal|\}$ decreases. The reason for this behavior is that we set $\gamma_i = 0$ for $i \in \Jcal$, which allows us to avoid unnecessary bias. The reason we can set
$\gamma_i = 0$ if $i \in \Jcal$ is because of the
following lemma:

\begin{restatable}{relemma}{lemratioOTB}\label{lem:ratioOTB}
    Suppose that $n \geq 4$. Let $\zeta_i = 3 \left(1 + \frac{2 \left(3 \ninh + 27\ln(4K/\delta)\right)}{\max\{\gamma_i, \half \ninh - 9 \ln(4K/\delta)\}}\right)$ and $\xi = 1 + \frac{\sumK \gamma_i}{n}$. For all $t \in [n/2 + 1, \ldots, n+1]$, $i \in [K]$, and $\delta \in (0, 1)$, with probability at least $1 - \delta$,
    \begin{align*}
        \frac{\pie}{\pti} \leq \zeta_i \xi\,, && \frac{\pti}{\ppi} \leq \zeta_i\,, && \frac{\ppi}{\pti} \leq 1 + \zeta_i \xi\,, && \frac{\pie}{\ppi} \leq \xi\,, && \frac{\nin + \gamma_i}{\ninh + \gamma_i} \leq 6 + \zeta_i\,.
    \end{align*}
\end{restatable}
Lemma~\ref{lem:ratioOTB} allows us to control the density ratio between
(a close approximation of) $\p$ and our estimators. We use
Lemma~\ref{lem:ratioOTB} for two purposes. The first purpose is to
control the range of the excess loss in order to apply a concentration
inequality (a version of Bernstein's inequality, Lemma~\ref{lem:bernie})
to relate $\KL(\p\|\potb)$ to the regret of an online learning
algorithm. \citet{vanderHoeven2023high} use a different OTB conversion
but use the same concentration inequality to relate $\KL(\p\|\hat p)$ to
a different regret, which ultimately leads to a $O((K +
\log(n)\log(1/\delta))/n)$ bound. The $\log(n)$ term in their bound does
not come from the online-to-batch conversion but from a weaker control
on the range of $\log\left(\frac{\pie}{\pti}\right)$
than we obtain from Lemma~\ref{lem:ratioOTB}. The second purpose of Lemma~\ref{lem:ratioOTB} is to relate different divergences to each other:
\begin{restatable}{relemma}{lemdivergenceequivalence}\label{lem:divergenceequivalence}
    For any pair $p, q \in \triangle^K$, suppose that $\half \leq \frac{p(i)}{q(i)} \leq 2$ for all $i \in [K]$. 
    Then
    \begin{align*}
        \tfrac{1}{6} \chi^2 (p,q) \leq \tfrac{1}{4} \chi^2 (q,p)  \leq \KL(p \| q)
        \leq \tfrac{5}{2} H^2(p ,q) \leq \tfrac{5}{2} \KL(q\|p)  \leq \tfrac{5}{2} \chi^2(q,p) \leq 5 \chi^2(p,q) \,.
    \end{align*}
\end{restatable}
The proof of Lemma~\ref{lem:divergenceequivalence} can be found in
Appendix~\ref{app:upperbounds}. Lemma~\ref{lem:divergenceequivalence}
tells us that, if the density ratio between two distributions is bounded
by a constant, then several divergence measures are equivalent up to
constants. This is useful because for several of these divergences we
already have optimal high-probability guarantees. Now,
Lemma~\ref{lem:ratioOTB} tells us that, with high
probability, as long as either all $\ninh$ are sufficiently big or we add
sufficient bias in the form of $\gamma_i$, the density ratio between (a
close approximation of) $\p$ and the add-$\gamma_i$ estimators $p_t$ is
bounded. Therefore, if we would have set $\gamma_i$ sufficiently high,
we could have used Lemma~\ref{lem:divergenceequivalence} and prior
results for these other diverges to control $\KL(\p\|p_{n+1})$. However, setting $\gamma_i$ too high introduces too much bias
to obtain minimax rates. Instead, our analysis involves carefully
balancing control of the density ratio and the amount of bias that is
introduced. In the remainder of this section we prove Theorems~\ref{th:meanrocks} and \ref{th:OTB}.

\subsection{Proof of Theorem~\ref{th:meanrocks}}\label{sec:gammaiestimator}
By Lemma~\ref{lem:ratioOTB} we have that for all $i \in [K]$, with probability at least $1 - \delta$
\begin{align}\label{eq:eventYB}
    \frac{\ppi}{p_{n+1}(i)} \leq \left(1 + \frac{\sumK \gamma_i}{n}
    \right)\left(\frac{3}{4} + \frac{9\ln(2K/\delta)}{\max\{\nin,
    \gamma_i\}}\right) =: \beta_i\,.
\end{align}
Let $\beta = \max_i \beta_i$. The inequality in \eqref{eq:eventYB} is very useful, as it allows us to control the ratio between $\ppi$ and $p_{n+1}(i)$. We use this to relate the $\KL$-divergence to the squared Hellinger distance, which we know how to control from prior work.
Specifically, by Lemma~4 of  \citet{yang1998asymptotic}, on event \eqref{eq:eventYB},
\begin{align*}%
    \KL(\p\|p_{n+1}) \leq (2 + \log(\beta))H^2(\p,p_{n+1})\,,
\end{align*}
where $H^2$ is the squared Hellinger distance. By Lemma~\ref{lem:helltohell} we have $H^2(\p,p_{n+1}) - H^2(\p,\bar p_n) 
    \leq \frac{\sumK \gamma_i}{2n}$. 
We also have $ H^2(\p,\bar p_{n+1}) = H^2(\bar p_{n+1},\p) \leq \KL(\bar{p}_{n+1}\|\p)$ \citep{gibbs2002choosing}.
Combining results from \citet{agrawal2022finite} and
\citet{paninski2003estimation} we obtain Lemma~\ref{lem:kl_simple_mean},
which tells us that, with probability at least $1 - \delta$, we have 
\begin{align*}
    \KL(\bar{p}_{n+1}\|\p) \leq \E[\KL(\bar{p}_{n+1}\|\p)] + \frac{6K + 6\ln(1/\delta)}{n} \leq \frac{7K + 6\ln(1/\delta)}{n}\,.
\end{align*}
Thus, with probability at least $1 - \delta$,
\begin{align}\label{eq:midpointTH2}
    \KL(\p\|p_{n+1})
    & \leq (2 + \log(\beta)) \Big(\frac{7K + 6\ln(1/\delta)}{n} + \frac{\sumK \gamma_i}{2n}\Big)\,.
\end{align}
If $J \leq \ln(K/\delta)$, for $i \in \mathcal{J}$ we have $\gamma_i = \frac{\log(K/\delta)}{J}$, for $i \not \in \mathcal{J}$ we have $\gamma_i = 0$, and therefore $\frac{\sumK \gamma_i}{n} =
\frac{\log(K/\delta)}{n} \leq \frac{K + \log(1/\delta)}{n}$. As a consequence, 
\begin{align*}
  \beta & = \max_i \Big\{\left(1 + \frac{\sumK \gamma_i}{n} \right)\left(\frac{3}{4} + \frac{9\ln(2K/\delta)}{\max\{\nin, \gamma_i\}}\right)\Big\} \\
  & \leq \max_i \Big\{\left(\frac{6}{4} + \frac{18\ln(2K/\delta)}{\max\{\nin, \gamma_i\}}\right)\Big\} \leq 40J = 40\min\{\log(K/\delta), J\} \,,
\end{align*}
where in the last inequality we used that $\max\{\nin, \gamma_i\} \geq \ln(K/\delta)J^{-1}$ for all $i$. 
If on the other hand $J > \ln(K/\delta)$, then for $i \in \mathcal{J}$ we have $\gamma_i = 1$, for $i \not \in \mathcal{J}$ we have $\gamma_i = 0$, and therefore $\frac{\sumK \gamma_i}{n} =
\frac{J}{n} \leq \frac{K + \ln(1/\delta)}{n}$. Consequently,
\begin{align*}
   \beta & = \max_i \Big\{\left(1 + \frac{\sumK \gamma_i}{n} \right)\left(\frac{3}{4} + \frac{9\ln(2K/\delta)}{\max\{\nin, \gamma_i\}}\right)\Big\} \\
   & \leq \left(\frac{6}{4} + \frac{18\ln(2K/\delta)}{\max\{\nin, \gamma_i\}} \right)\leq 40\log(K/\delta) = 40\min\{\log(K/\delta), J\} \,,
\end{align*}
where in the last inequality is due to the fact that $\max\{\nin, \gamma_i\} \geq 1$ for all $i$. Thus, if $|\Jcal| > 0$ then the bounds on $\beta$ combined with \eqref{eq:midpointTH2} lead to the conclusion that with probability at least $1 - \delta$
\begin{align*}
    \KL(\p\|p_{n+1})
    & \leq (2 + \log(40\min\{\log(K/\delta), J\})) \Big(\frac{9K + 9\ln(1/\delta)}{n}\Big)\,.
\end{align*}
Finally, if $|\Jcal| = 0$, then $\frac{\sumK \gamma_i}{n} = 0$ and 
\begin{align*}
    \beta & = \max_i \Big\{\left(1 + \frac{\sumK \gamma_i}{n} \right)\left(\frac{3}{4} + \frac{9\ln(2K/\delta)}{\max\{\nin, \gamma_i\}}\right)\Big\} \\
    & = \max_i \Big\{\frac{3}{4} + \frac{9\ln(2K/\delta)}{\nin}\Big\} \leq 2\,,
\end{align*}
where the last inequality is due to the fact that $\nin \geq 32 \ln(K/\delta)$ for all $i$. Combined with \eqref{eq:midpointTH2} this leads to the conclusion that with probability at least $1 - \delta$
\begin{align*}
    \KL(\p\|p_{n+1})
    & \leq (2 + \log(2)) \Big(\frac{7K + 6\ln(1/\delta)}{n}\Big)\,,
\end{align*}
which concludes the proof. 

\subsection{Proof of Theorem~\ref{th:OTB}}
Let $\vartheta = \left(1 + \frac{\sumK \gamma_i}{n} \right)$. We denote by $\mathcal{Z}$ the event that for all $i \in K$ and $t\in [n/2 + 1,\ldots, n+1]$ 
\begin{align}\label{eq:eventOTB}
        & \frac{\pie}{\pti} \leq \vartheta\left(\frac{3}{4} + \frac{9\ln(2K/\delta)}{\max\{\ninh, \gamma_i\}}\right) 
        && \frac{\pti}{\ppi} \leq 3 \left(1 + \frac{2 \left(3 \ninh + 27\ln(4K/\delta)\right)}{\max\{\gamma_i, \half \ninh - 9 \ln(4K/\delta)\}}\right) \notag\\
        & \frac{\ppi}{\pti} \leq 1 + \vartheta\left(\frac{3}{4} + \frac{9 \ln(4K/\delta)}{\max\{\ninh, \gamma_i\}}\right) 
        && \frac{\nin + \gamma_i}{\ninh + \gamma_i} \leq 6 + \frac{18 \ln(2K/\delta)}{\ninh + \gamma_i} \,.
\end{align}
By Lemma~\ref{lem:ratioOTB} we have that $\Pp(\mathcal{Z}) \geq 1 -
\delta$ and by definition we have that $\frac{\pie}{\ppi} \leq 1 +
\frac{\sumK \gamma_i}{n}$. To simplify notation let $\beta_i = \vartheta\left(\frac{3}{4} + \frac{9\ln(2K/\delta)}{\max\{\nin, \gamma_i\}}\right)$ 
and $\beta = \max_i \beta_i$.

At this point we can use Lemma~\ref{lem:KLtoReg} to relate $\KL(\p\|\potb)$ to the regret of an online learning algorithm on the second half of the data, which tells us that on event $\mathcal{Z}$
\begin{align*}
     \KL(\p\|\potb) \leq \frac{2}{n}\Big(2\regret_T + 2 \sumK \gamma_i + (4 + 2 \log(\beta))\log(1/\delta)\Big)\,,
\end{align*}
where $\regret_T = \sumnhp \left( -\ln{p_t(X_t)} - (- \ln{\pp(X_t)} )\right)$. The result of Lemma~\ref{lem:KLtoReg} is very useful, as we can now rely on standard regret bounds to separate the bound into several parts which are easier to control. The result of Lemma~\ref{lem:KLtoReg} follows from first applying Freedman's inequality. However, a na\"{i}ve application of Freedman's inequality would lead to a vacuous bound, as $|\ln(\pie/\pti)|$ is unbounded. We instead rely on \eqref{eq:eventOTB} to control said ratio before applying Freedman's inequality. Next, we still need to control the variance term of Freedman's inequality. For this we generalize the proof of Lemma 4 by \citet{yang1998asymptotic} and show that the cumulative variance is bounded by $\regret_T$, after which the result of Lemma~\ref{lem:KLtoReg} follows with some computations.

We continue by bounding the regret (Lemma~\ref{lem:regretbound}):
\begin{align*}
    \regret_T & \leq \sumK \gamma_i \log(\pnhpi/\ppi) +  2\sumK \gamma_i + \frac{n}{2} \KL(\bar{p}_{n/2+1}\|\p)  + \sumK \log\left(\frac{\nin+ \gamma_i}{\ninh+\gamma_i}\right)\,.
\end{align*}
Note that this is not a standard regret bound for prediction with log loss, as that would lead to a $O(\log(n))$ term (see e.g.\ Chapter~9 of \citep{cesa2006prediction}). Instead, we carefully use suffix averaging. 
When combined with the above, we can see that on event $\mathcal{Z}$
\begin{align*}
    \KL(\p\|\potb) & \leq \frac{2}{n}\Big(2\sumK \gamma_i \log(\pnhpi/\ppi) +  6\sumK \gamma_i + n \KL(\bar{p}_{n/2+1}\|\p) \\
	& \quad + 2 \sumK \log\left(\frac{\nin+ \gamma_i}{\ninh+\gamma_i}\right)  + (4 + 2 \log(\beta))\log(1/\delta)\Big)\,.
\end{align*}
The term $\KL(\bar{p}_{n/2+1}\|\p)$ can be controlled with standard results: Lemma~\ref{lem:kl_simple_mean} tells us that with probability at least $1 - \delta$
\begin{align*}
    \KL(\bar{p}_{n/2+1}\|\p) \leq \frac{14K + 12\ln(1/\delta)}{n}\,.
\end{align*}
The remaining challenge is to control
\begin{align*}
    A = 2\sumK \gamma_i \log(\pnhpi/\ppi) +  6\sumK \gamma_i + 2 \sumK \log\left(\frac{\nin+ \gamma_i}{\ninh+\gamma_i}\right)  + (4 + 2 \log(\beta))\log(1/\delta)\,.
\end{align*}
Here, the main challenge lies in controlling $2 \sumK \log\left(\frac{\nin+ \gamma_i}{\ninh+\gamma_i}\right)$. We cannot na\"ively rely on the fact that on event $\mathcal{Z}$ we have that $\frac{\nin+ \gamma_i}{\ninh+\gamma_i} \leq 6 + 18\frac{\ln(K/\delta)}{\gamma_i}$, as this could potentially lead to a $K\ln(\ln(1/\delta))$ term in the regret bound. Instead, we will split the analysis in two cases. In the first case  $\frac{\log(K/\delta)}{J} > 1$ and by definition of $\gamma_i$ we can in fact use the na\"ive bound, which we do to prove Lemma~\ref{lem:caseJ1}:
\begin{align*}
    A \leq 14\log(K/\delta) +  40 K + 4\log(1/\delta)\log\left(400J\right)\,.
\end{align*}
In the second case $\frac{\log(K/\delta)}{J} \leq 1$. In this case, we carefully control $A$ in Lemma~\ref{lem:caseJ2}, which tells us that with probability at least $1-3\delta$
\begin{align*}
    A \leq 100\log(1/\delta) \log(800J) + 4\Jt\log\left(24\left(\log\left(\frac{\Jt}{\log(1/\delta)}\right) \vee 1 \right)\right) + 20 K \log(100)\,.
\end{align*}
All combined, we can see that with probability at least $1-5\delta$
\begin{align*}
     \KL(\p\|\potb) %
        & \leq \frac{1}{n}\Big(250\log(1/\delta) \log(800J) + 8\Jt\log\left(24\left(\log\left(\frac{\Jt}{\log(1/\delta)}\right) \vee 1\right)\right) + 200K\Big)\,,
\end{align*}
where we coarsely bounded all the constants.

\section{Lower Bounds}\label{sec:lowerbounds}

In this section we prove our main lower bound from
\eqref{eqn:minimaxLowerUpper}: 
\begin{restatable}{retheorem}{thintroLB1}\label{th:introLB1}
    Let $n > K^2$ and $n > \tfrac{4}{3}\log(1/\delta)$ for any $\delta
    \in (0,\half)$. Then
    \begin{align*}
        \inf_{\phat} \sup_{\p} ~
        \Pr_{\p}\left(\KL(\p\|\phat) \geq C \frac{\max\{ K, \log(K)\log(1/\delta)\}}{n}\right) > \delta,
    \end{align*}
    where $C > 0$ is an absolute constant.
\end{restatable}
In Section~\ref{sec:LBinstance} we first provide an
estimator-dependent lower bound by direct calculations. Then we
establish the minimax lower bound from Theorem~\ref{th:introLB1} by
breaking it up into two parts: in Section~\ref{sec:minimaxLowerBound1} we first prove a lower bound of order
$\log(K)\log(1/\delta)/n$ by building on the intuition developed in the
direct calculations. Because standard techniques based on a
reduction to hypothesis testing are not precise enough to capture the
$\log(K)$ factor, this requires a novel reduction to what we call a
\emph{weak testing} problem (Section~\ref{sec:standardreductioninsufficient}). 
For the second part of the proof of Theorem~\ref{th:introLB1}, we provide a lower bound
of order $K/n$ in Section~\ref{sec:minmaxlowerboundII}, for which standard techniques based on Fano's inequality
suffice after restricting the model to a ball around the uniform
distribution. The missing proofs in this section can be found in Appendix~\ref{app:lowerbounds}.

\subsection{Estimator-dependent Lower Bound}\label{sec:LBinstance}

\begin{restatable}{retheorem}{thmDeltaDependentlb}\label{thm:DeltaDependent_lb}
Suppose that $n > \frac{4}{3}\log(1/\delta)$. Denote by $\phat^0 \in
\model$ the output of an estimator $\phat$ on the sample $X_1 = \cdots =
X_n = 1$. Then, for any $\phat$, there exists $\p
\in \model$ such that
\begin{align*}
  \Pp_{\p}\Big(\KL(\p\|\phat) & \geq \frac{2\ln(1/\delta)}{3n}\left(\ln\left(\frac{2\ln(1/\delta)(K-1)}{3n \sum_{i = 2}^{K}\phat^0(i)}\right) - 1 \right) + \sum_{i = 2}^{K}\phat^0(i)\Big) > \delta\,.
\end{align*}
\end{restatable}
We can see that any estimator that minimizes the lower bound satisfies
$\sum_{i = 2}^{K}\phat^0(i) \propto \frac{\log(1/\delta)}{n}$, for which
the lower bound becomes of order $\frac{\ln(K)\ln(1/\delta)}{n}$, and no
estimator that does not depend on~$\delta$ can match this rate.
In particular, if $\phat$ is an add-$\gamma$ estimator, so $\phat(i) =
\frac{\nin+\gamma}{n+K\gamma}$, the lower bound specializes to
\[
    \frac{2\ln(1/\delta)}{3n}\left(\ln\left(\frac{2\ln(1/\delta)(n + \gamma K)}{3n\gamma}\right) - 1 \right) + \frac{\gamma}{n + \gamma K},
\]
so $\gamma \propto \ln(1/\delta)/K$ will achieve the
$\frac{\ln(K)\ln(1/\delta)}{n}$ rate, but the KT
($\gamma=1/2$) and Laplace ($\gamma=1$) estimators lead to a 
rate of $\frac{\ln(\ln (1/\delta))\ln(1/\delta)}{n}$, which is worse in
the regime where $\ln(K)\ln(1/\delta)$ dominates $K$.

\subsection{Minimax Lower Bound, Part I: $\log(K)\log(1/\delta)/n$}\label{sec:minimaxLowerBound1}

Inspired by the proof of the estimator-dependent lower
bound, we will lower bound the minimax rate by restricting the supremum
in \eqref{eqn:worst_case_rate} to a finite set of probability mass
functions $p_2,\ldots,p_K \in \model$ of the form
\begin{align}\label{eqn:hard_dists}
  p_j(i) = \begin{cases}
      1 - \frac{\alpha}{n} & \text{if $i = 1$}\\
      \frac{\alpha}{n} & \text{if $i = j$}\\
      0 & \text{otherwise,}
  \end{cases} && \text{~ where we assume that ~} \alpha := \frac{2}{3}\log(1/\delta) \leq \half n\,.
\end{align}

\subsubsection{Insufficiency of the Standard Reduction to Hypothesis
Testing}\label{sec:standardreductioninsufficient}

The standard approach to establishing lower bounds on the minimax rate
goes via a general reduction to hypothesis testing \citep{Tsybakov2009}.
It is commonly applied to metrics, but may easily be adjusted to KL
divergence as shown in the following lemma. Working directly with KL
divergence is potentially tighter than first lower-bounding KL by a
metric. Since the lemma is not restricted to discrete distributions, we
state it in terms of the KL divergence between general distributions,
which is $\KL(P \| Q) = \int \ln\big(\frac{\der P}{\der Q}\big) \der P$ if $P \ll
Q$ and equals infinity otherwise.
\begin{restatable}{relemma}{lemTsybakovReductionForKL}\label{lem:TsybakovReduction_for_KL}
Let $P_1,\ldots,P_M$ be distributions defined on a sample space
$\domain$, and let $P_j^n$ be the distribution of $n$ independent
draws from $P_j$. Given any $\delta \in [0,1]$, suppose that
\begin{align}
  \forall j \neq k: \quad \KL\Big(P_j \| \frac{P_j + P_k}{2}\Big)
    &\geq s_n,\label{eqn:pairwise_separation}\\
  \inf_\Psi \max_j P_j^n(\Psi \neq j) &> \delta, \label{eqn:test_fails}
\end{align}
where the infimum is over all possible hypothesis tests $\Psi :
\domain^n \to \{1,\ldots,M\}$. Then
\[
  \inf_{\hat P} \max_j P_j^n\Big(\KL(P_j \| \hat P) \geq s_n\Big) > \delta,
\]
where the infimum is over all estimators based on $n$ observations.
\end{restatable}
In our context the sample space is $\domain = [K]$ and the conclusion
of the lemma gives a lower bound on the minimax rate of $r_n^*(\delta)
\geq s_n$. So what is the best lower bound we can hope for using
Lemma~\ref{lem:TsybakovReduction_for_KL} if we apply it to the
distributions in \eqref{eqn:hard_dists}? Since
\[
  \KL\Big(p_j \| \frac{p_j + p_k}{2}\Big)
  = \frac{\alpha}{n} \ln \frac{\alpha/n}{\half \alpha/n}
  = \frac{\alpha \ln(2)}{n}
  \qquad \text{for any $j \neq k$,}
\]
condition \eqref{eqn:pairwise_separation} requires that $s_n \leq
\frac{\alpha \ln(2)}{n} = \frac{2\ln(2)}{3}\frac{\log(1/\delta)}{n}$,
which falls short of the $\frac{\log(K)\log(1/\delta)}{n}$ rate that we
are trying to show. Consequently, this standard reduction is
insufficient to recover the right lower bound.

\subsubsection{Reduction to Weak Hypothesis Testing}

Ordinary hypothesis testing among distributions $P_1,\ldots,P_M$ is hard
when there exist two distributions that cannot be properly
distinguished. But the distributions corresponding to \eqref{eqn:hard_dists} are
indistinguishable in a stronger sense: they all assign probability
larger than $\delta$ to the same sequence $X_1 = \cdots = X_n = 1$
consisting of only ones, and as a consequence they are all
indistinguishable simultaneously.

In order to capture this stronger sense of indistinguishability, we
introduce the easier task of \emph{weak hypothesis testing}, where the goal
is not to identify the true distribution among $P_1,\ldots,P_M$ but to
identify a set of $M-1$ distributions that contains the true
distribution. A weak hypothesis test is defined as a measurable function
$\Psi : \domain^n \to [M]$, which aims to identify a single distribution
$P_\Psi$ that is believed \emph{not} to be the true distribution. In
other words, the test $\Psi$ is correct if the set $\{P_j | j \neq
\Psi\}$ contains the true distribution, and makes a mistake if $\Psi =
j$ when $P_j$ is true.

\begin{restatable}{relemma}{lemWeakTesting}\label{lem:weak_testing}
Let $P_1,\ldots,P_M$ be distributions defined on a sample space
$\domain$, and let $P_j^n$ be the distribution of $n$ independent
samples from $P_j$. Given any $\delta \in [0,1]$, suppose that
\begin{equation}
  \inf_\Psi \max_j P_j^n(\Psi = j) > \delta, \label{eqn:weak_test_fails}
\end{equation}
where the infimum is over all possible weak hypothesis tests $\Psi :
\domain^n \to \{1,\ldots,M\}$. Then
\begin{align}\label{eqn:weak_test_lower_bound}
    \inf_{\hat P} \max_j P_j^n\Big(\KL(P_j \| \hat P) \geq s_n\Big) >
    \delta && \text{~ for ~} s_n = \inf_P \max_j \KL(P_j \| P).
\end{align}
\end{restatable}

Condition \eqref{eqn:weak_test_fails} is analogous to
\eqref{eqn:test_fails} in that it lower bounds the worst-case
probability that the (weak) hypothesis test fails. Applying this lemma
to the distributions corresponding to \eqref{eqn:hard_dists}, we obtain
the desired lower bound:
\begin{restatable}{retheorem}{thmLogKLowerBound}\label{thm:logK_lower_bound}
  For $K \geq 2$ and any $\delta \in (0,1]$, the minimax rate is at least 
  \[
    r_n^*(\delta)
      \geq \frac{2 \log (K-1) \log(1/\delta)}{3 n}
    \qquad
    \text{for all $n > \frac{4}{3}\log(1/\delta)$.}
  \]
\end{restatable}

\subsection{Minimax Lower Bound, Part II: $K/n$}\label{sec:minmaxlowerboundII}
We now turn to proving a lower bound of order $K/n$, which is the
standard parametric rate for a model with $K-1$ free parameters.
Nevertheless, obtaining the right rate is not entirely straightforward,
because it requires identifying a suitable subset $\model_0$ of $\model$
for which KL covering numbers and variational distance packing numbers
can be shown to line up appropriately.

By Pinsker's inequality $\frac{1}{2} V(p,q)^2 \leq \KL(p\|q)$ for any
$p$ and $q$,
so it is sufficient to prove a lower bound on the minimax rate
for total variation. Although it is well known that such a lower bound
can be obtained using standard techniques (see \citep{canonne2020short}
or \citep[Exercise VI.8]{polyanskiy2025}), we spell out the details
explicitly. In particular, since total variation is a metric, we can
build on a result by \citet{YangBarron1999}, which is itself an elegant
application of the standard reduction to hypothesis testing
\citep{Tsybakov2009} combined with Fano's inequality to lower bound the
hypothesis testing error. Specialized to our setting, Yang and Barron's
result gives the following:

For any $\model_0 \subset \model$ and $\epsilon > 0$, let
$N(\model_0,\epsilon,\textnormal{KL})$ be the Kullback-Leibler
$\epsilon^2$-entropy of $\model_0$\footnote{That is, $\log m$ for the
  smallest $m$ for which there exist mass functions $p_1,\ldots,p_m \in
  \model$ such
  that for every $p \in \model_0$ there exists a $j$ for which
$\KL(p\|p_j) \leq \epsilon^2$.}, and let $M(\model_0,\epsilon,V)$ be the
variational distance $\epsilon$-packing entropy of
$\model_0$\footnote{That is, $\log m$ for the largest $m$ such that
there exist $p_1,\ldots,p_m \in \model_0$ with pairwise distance
$V(p_j,p_{j'}) > \epsilon$ for all $j \neq j'$.}.
\begin{theorem}[\citealp{YangBarron1999}]\label{thm:YangBarron}
  For any non-increasing, right-continuous bounds $N(\epsilon) \geq
  N(\model_0,\epsilon,\textnormal{KL})$ and $M(\epsilon) \leq
  M(\model_0,\epsilon,V)$, let $\epsilon_n,\loweps_n > 0$ be such that $\epsilon_n^2 = \frac{N(\epsilon_n)}{n}$ and $M(\loweps_n) \geq 4n\epsilon_n^2 + 2 \log 2.$
  Then
  \[
    \inf_{\hat p} \sup_{\p \in \model_0} \Pr_{\p}\Big(V(\p,\hat p) \geq
    \frac{\loweps_n}{2}\Big) \geq \frac{1}{2},
  \]
  where the infimum is over all estimators $\hat p$ based on $n$
  observations.
\end{theorem}
For the full model $\model_0 = \model$, \citet{Tang2022} shows that
$N(\model,\epsilon,\textnormal{KL}) \leq \frac{K-1}{2}\log \frac{800 \log
K}{\epsilon^2}$ for $\epsilon^2 \leq \log K$, but this does not lead to
the right parametric rate $\epsilon_n^2 \approx K/n$. The solution is to
restrict the model to a ball of appropriate radius $\alpha \in
\big(0,\frac{1}{2K}\big]$ around the uniform distribution with
probabilities $u(i) = 1/K$:
\[
  \model_0 = \{p \in \model: \|p - u\|_2 \leq \alpha\},
\]
where $\|p - u\|_2 = \sqrt{\sum_{i=1}^K (p(i) - u(i))^2}$. This restriction
ensures that, for all $p \in \model_0$,
\[
  \Big(p(i) - \frac{1}{K}\Big)^2
    \leq \|p - u\|_2^2 \leq \alpha^2 \leq \Big(\frac{1}{2K}\Big)^2
  \qquad \Longrightarrow \qquad p(i) \in \Big[\frac{1}{2K}, \frac{3}{2K}\Big],
\]
so the ratios between any two $p,q \in \model_0$ are uniformly
bounded.
 We will end up tuning $\alpha \approx
\frac{1}{\sqrt{n}}$, so $\model_0$ is actually shrinking with $n$. We
obtain the following bounds on the covering and packing entropies:
\begin{restatable}{relemma}{lemCover}\label{lem:covering_numbers}
  For any $\alpha \in \big(0,\frac{1}{2K}\big]$ and $\epsilon > 0$,
  the covering and packing entropies are bounded by
  \begin{align*}
    N(\model_0,\epsilon,\textnormal{KL})
      &\leq K \log \Big(\frac{\alpha 2\sqrt{2K}}{\epsilon} + 1\Big),
      &
    M(\model_0,\epsilon,V)
      &\geq K \log \frac{\alpha}{\epsilon}
         + \frac{K}{2} \log \frac{K \pi}{8}.
  \end{align*}
\end{restatable}

Using these in Theorem~\ref{thm:YangBarron}, we obtain the following
result:
\begin{restatable}{retheorem}{ThmParametricLowerBound}\label{thm:parametric_lower_bound}
  There exists an absolute constant $C > 0$ such that,
  for any $K \geq 2$ and $\delta \in (0,1/2)$, the minimax rate is at least 
  \[
    r_n^*(\delta)
      \geq \frac{C K}{n}
    \qquad
    \text{for all $n \geq \tfrac{\log 2}{2}K^2$.}
  \]
\end{restatable}
The proof shows that $C = 4 \times 10^{-6}$ will work, but presumably this is very
far from tight.

\section{Conclusion}\label{sec:conclusion}

We presented nearly minimax rates for discrete distribution estimation
in KL divergence with high probability. Our results represent an
improvement in the understanding of arguably the simplest excess risk
minimization problem with unbounded losses. An intriguing direction for
future work is to extend the ideas we have presented here to more
involved but related problems with unbounded losses such as prediction
of Markov chains (see \citep{han2021optimal}) and logistic regression.
For that purpose, it is worthwhile to consider the approach of
\citet{mourtada2025estimation}, whose estimator is simpler than the OTB
estimator. Given that the approaches are technically different, both can
be useful for generalizations to more involved settings. Finally, we did
not carefully optimize our constants, so there is room for improvement
there. Considering the work of \citet{braess2004bernstein}, there is
substantial interest in understanding and obtaining the exact constant
factors on the dominant term in the minimax rate.

\acks{We thank Nikita Zhivotovskiy for bringing \citep{mourtada2025estimation} to our attention. 
Olkhovskaia and Van Erven were supported by the Netherlands Organization for Scientific
Research (NWO) under grant numbers VI.Veni.232.068 and VI.Vidi.192.095, respectively.}
\bibliography{sample.bib}

\newpage
\appendix

\section{Asymptotic Behavior of the Maximum Likelihood Estimator}
\label{app:MLEasymp}

We formalize here the claim from the introduction that $n
\KL(\p\|\pml)$ converges to a $\chi^2$ distribution as $n\to \infty$
with $\p$ held fixed.

\begin{theorem}\label{thm:MLEWeakConvergence}
  Let $\pml$ denote the maximum likelihood estimator, and let $\p$ be
  arbitrary with support of size $M = |\{ i \mid \p(i) > 0\}|$. The
  case $M=1$ is trivial, because then $\KL(\p \| \pml) = 0$ almost
  surely, so assume $M \geq 2$. Then
  \begin{equation}\label{eqn:MLEWeakConvergence}
    2 n \KL(\p \| \pml) ~\weaklyconvergesto~ \chi_{M-1}^2.
  \end{equation}
  It follows that
  \begin{equation}\label{eqn:ChiSquaredConcentration}
    \KL(\p \| \pml) \leq \frac{M-1 + \frac{3}{2}\ln(2/\delta)}{n}
    \qquad \text{with $\p$-probability at least $1-\delta$}
  \end{equation}
  for all sufficiently large $n$.
\end{theorem}

\begin{proof}
  Our approach will be to apply the second-order delta method
  \citep[Section~3.3]{VanDerVaart1998}, which combines the fact that the
  MLE is asymptotically Gaussian (by the central limit theorem) with a
  second-order Taylor expansion of the KL divergence in its second
  argument to get the desired conclusion. 

  We assume without loss of generality that $K = M$ (i.e.\ $\p$ has
  full support), because the distribution of $\KL(\p \| \pml)$ does not
  change if we remove all outcomes $i$ for which $\p(i) = 0$. Then,
  because of the simplex constraint, the MLE will be asymptotically
  Gaussian on a subspace of dimension $K-1$. It is therefore technically
  convenient to let $\mu, \bar \mu \in \reals^{K-1}$ denote the vectors
  with the first $K-1$ coordinates of $\p$ and $\pml$, respectively:
  $\mu_i = \p(i)$ and $\bar \mu_i = n_i/n$ for $i=1,\ldots,K-1$. Then
  we may view $\bar \mu$ as
  \[
    \bar \mu = \frac{1}{n} \sum_{t=1}^n e_{X_t},
  \]
  where $e_{X_t}$ is the standard basis vector in direction $X_t$ if
  $X_t \neq K$ and the all-zeros vector otherwise. Therefore, by the
  central limit theorem, the difference between $\bar \mu$ and $\mu$
  converges in distribution to a normal distribution:
  \begin{equation}\label{eqn:MLECentralLimit}
    \sqrt{n} (\bar \mu - \mu) \weaklyconvergesto \Sigma^{1/2} Z,
  \end{equation}
  where $Z \sim \normaldist(0,I)$ is standard normal in $K-1$ dimensions
  and $\Sigma \in \reals^{(K-1)\times(K-1)}$ is the covariance matrix of
  $e_X$:
  \[
    \Sigma_{ij}
    = \E[(\indicator[X=i] - \mu_i)(\indicator[X=j] - \mu_j)]
    = 
      \begin{cases}
        \mu_i (1-\mu_i) & \text{if $i = j$,}\\
        -\mu_i\mu_j & \text{otherwise}.
      \end{cases}
  \]
  By the continuous mapping theorem, it follows that
  \begin{equation}\label{eqn:QuadnormConverges}
    n \|\bar \mu - \mu\|^2 \weaklyconvergesto Z^\top \Sigma Z.
  \end{equation}
  Given the one-to-one relation between $\pml$ and $\bar \mu$, we may
  consider $\KL(\p \| \pml)$ as a function $\bar \mu$:
  \[
    \KL(\p \| \pml) = \phi(\bar \mu) :=
      \sum_{i=1}^{K-1} \mu_i \ln \frac{\mu_i}{\bar \mu_i}
      + f_K(\mu) \ln \frac{f_K(\mu)}{f_K(\bar \mu)},
  \]
  where we have abbreviated $f_K(\mu) = 1 - \sum_{i=1}^{K-1} \mu_i$.
  Since $\mu_i > 0$ for all $i$ by assumption, $\phi$ is twice
  differentiable for all $\bar \mu$ close enough to $\mu$, with
  \begin{align*}
    \nabla \phi(\bar \mu)_i
      &= \frac{f_K(\mu)}{f_K(\bar \mu)} - \frac{\mu_i}{\bar \mu_i},
      &
    \nabla^2 \phi(\bar \mu)_{ij}
      &= \begin{cases}
        \frac{f_K(\mu)}{f_K(\bar \mu)^2} + \frac{\mu_i}{(\bar \mu_i)^2} & \text{if $i = j$,}\\
        \frac{f_K(\mu)}{f_K(\bar \mu)^2} & \text{otherwise.}
      \end{cases}
  \end{align*}
  Let $H = \nabla^2 \phi(\mu)$ denote the Hessian of $\phi$ at $\bar \mu
  = \mu$, i.e.\ the Fisher information. Then, because $\phi(\mu) = 0$
  and $\nabla \phi(\mu) = 0$, a second-order Taylor expansion of $\phi$
  around $\bar \mu = \mu$ gives
  \begin{align*}
    \phi(\bar \mu)
       &= \phi(\mu) + (\bar \mu - \mu) \nabla \phi(\mu)
        + \half (\bar \mu - \mu)^\top H (\bar \mu -
        \mu)
        + o_P\big(\|\bar \mu - \mu\|^2\big)\\
       &= \half (\bar \mu - \mu)^\top H (\bar \mu -
       \mu) + o_P\big(\|\bar \mu - \mu\|^2\big)\\
    2n\KL(\p \| \pml) = 2n \phi(\bar \mu)
       &= n (\bar \mu - \mu)^\top H (\bar \mu -
       \mu) + o_P\big(n \|\bar \mu - \mu\|^2\big)
       \weaklyconvergesto Z^\top \Sigma^{1/2} H \Sigma^{1/2} Z,
  \end{align*}
  where convergence in the last step holds because both terms converge:
  first, \eqref{eqn:MLECentralLimit} and the continuous mapping theorem
  imply that
  \[
    n (\bar \mu - \mu)^\top H (\bar \mu - \mu)
    \weaklyconvergesto 
    Z^\top \Sigma^{1/2} H \Sigma^{1/2} Z.
  \]
  And, second, since $n \|\bar \mu - \mu\|^2$ converges in distribution
  by \eqref{eqn:QuadnormConverges}, it is uniformly tight by Prohorov's
  theorem \citep{VanDerVaart1998}, which implies that $o_P\big(n \|\bar
  \mu - \mu\|^2\big) = o_P(1)$. The convergence step above therefore
  follows from Slutsky's theorem, which allows us to combine the 
  convergence of the two terms.

  It remains to simplify $\Sigma^{1/2} H \Sigma^{1/2}$, using that the
  Fisher information matrix $H$ is the inverse of the covariance matrix
  $\Sigma$. Specifically, it can be confirmed that $H$ is the left
  inverse of $\Sigma$ by direct calculation:
  \begin{align*}
    (H \Sigma)_{ij}
    &= \sum_{k=1}^{K-1}
    \Big(
      \frac{1}{f_K(\mu)} + \indicator[k = i] \frac{1}{\mu_i}
    \Big)
    \big(
      \indicator[k = j] \mu_j -\mu_k\mu_j
    \big)\\
    &= \sum_{k=1}^{K-1}
    \Big(
      \indicator[k = j]\frac{\mu_j}{f_K(\mu)} 
      + \indicator[k = i = j] \frac{\mu_j}{\mu_i}
      - \frac{\mu_k\mu_j}{f_K(\mu)}
      - \indicator[k = i] \frac{\mu_k\mu_j}{\mu_i}
    \Big)\\
    &= \frac{\mu_j}{f_K(\mu)} 
      + \indicator[i = j]
      - \Big(\frac{\mu_j}{f_K(\mu)} - \mu_j\Big)
      - \mu_j
    = \indicator[i = j],
  \end{align*}
  and, since both $H$ and $\Sigma$ are symmetric, it follows that $H =
  \Sigma^{-1}$. Being the inverse of a symmetric, positive definite
  matrix, $\Sigma^{-1}$ is also positive definite, and therefore
  \[
    \Sigma^{1/2} H \Sigma^{1/2}
    = \Sigma^{1/2} \Sigma^{-1} \Sigma^{1/2}
    = \Sigma^{1/2} \Sigma^{-1/2} \Sigma^{-1/2} \Sigma^{1/2}
    = I.
  \]
  Putting everything together, we have shown that
  \[
    2 n\KL(\p \| \pml) \weaklyconvergesto Z^\top Z,
  \]
  where $Z^\top Z$ has a $\chi_{K-1}^2$ distribution. Since we have
  reduced to the case that $K=M$, this completes the proof of
  \eqref{eqn:MLEWeakConvergence}.

  To obtain \eqref{eqn:ChiSquaredConcentration}, it suffices to use
  concentration of the $\chi^2$ distribution around its mean. In
  particular, if $Y \sim \chi_{M-1}^2$, then \citet{LaurentMassart2000}
  show that
  \begin{align*}
    \Pr\Big(
    Y \geq M-1 + 2\sqrt{(M-1)\ln(1/\delta)} + 2\ln(1/\delta)
    \Big) &\leq \delta\\
    \Pr\Big( Y \geq 2(M-1) + 3\ln(1/\delta) \Big) &\leq \delta,
  \end{align*}
  where we have used that $\sqrt{ab} \leq \half a + \half b$ for
  $a,b\geq 0$. Since, for all sufficiently large $n$,
  \[
    \Big|
      \Pr\Big( 2n\KL(\p \| \pml) \geq 2(M-1) + 3\ln(2/\delta)
      \Big)
    -
      \Pr\Big( Y \geq 2(M-1) + 3\ln(2/\delta) \Big)
      \Big| \leq \frac{\delta}{2},
  \]
  we conclude that, for all such $n$,
  \[
    \Pr\Big( 2n\KL(\p \| \pml) \geq 2(M-1) + 3\ln(2/\delta) \Big)
    \leq 
    \Pr\Big( Y \geq 2(M-1) + 3\ln(2/\delta) \Big) + \frac{\delta}{2}
    \leq 
    \frac{\delta}{2} + \frac{\delta}{2}
    = \delta,
  \]
  from which \eqref{eqn:ChiSquaredConcentration} follows.
\end{proof}

\section{Auxilliary Results}

\begin{restatable}{relemma}{lemgenYangBarron}\label{lem:generalizedYangBarron}
  For any distributions $p,q,r$ such that $\frac{r(i)}{q(i)} \leq V$ for
  all $i$
  \[
    \sumK p(i) \Big(\log \frac{r(i)}{q(i)}\Big)^2 
    \leq (2+\log V) \sumK p(i) \log \frac{r(i)}{q(i)} + \sumK p(i) \frac{q(i)}{r(i)} - 1.
  \]
    
\end{restatable}
\begin{proof}
  We generalize the proof of Lemma~4 by Yang \& Barron, 1998: let
  $\phi(V) = \frac{\log V + 1/V - 1}{(\log V)^2} \geq \frac{1}{2 + \log
  V}$ be their decreasing function $\phi_1$. Then
  \begin{align*}
    \phi(V) \sumK p(i) \Big(\log \frac{r(i)}{q(i)}\Big)^2 
      &\leq \sumK p(i)\phi(\frac{r(i)}{q(i)}) \Big(\log \frac{r(i)}{q(i)}\Big)^2 \\
      &= \sumK p \Big(\log \frac{r(i)}{q(i)} + \frac{q(i)}{r(i)} - 1\Big)\\
      &= \sumK p(i) \log \frac{r(i)}{q(i)} + \sumK p(i) \frac{q(i)}{r(i)} - 1.
  \end{align*}
\end{proof}
Applying this with $p = \p$ and $r=\pp$ we obtain:
\begin{restatable}{recorollary}{corapplyYB}\label{cor:applyYB}
  For any $q$ such that $\log \frac{\ppi}{q(i)} \leq \alpha$ for all
  $i$,
  \[
  \sumK \pie \log \Big(\frac{\ppi}{q(i)}\Big)^2
  \leq (2 + \alpha) \Big(\sumK \pie\log\left(\frac{\ppi}{q(i)}\right)  + \frac{\sumK \gamma_i}{n}\Big).
  \]
\end{restatable}
\begin{proof}
  By Lemma~\ref{lem:generalizedYangBarron}
\begin{align*}
  \sumK \pie \log \Big(\frac{\ppi}{q(i)}\Big)^2
  &\leq (2 + \alpha) \Big(\sumK \pie\log\left(\frac{\ppi}{q(i)}\right) + \sumK \pie\frac{q(i)}{\ppi}
  - 1\Big)\\
  &= (2 + \alpha) \Big(\sumK \pie\log\left(\frac{\ppi}{q(i)}\right) + \sumK q(i) \frac{\pie}{\ppi}
   - 1\Big)\\
  &\leq (2 + \alpha) \Big(\sumK \pie\log\left(\frac{\ppi}{q(i)}\right) + \frac{n + \sumK \gamma_i}{n} -
  1\Big)\\
  &= (2 + \alpha) \Big(\sumK \pie\log\left(\frac{\ppi}{q(i)}\right) + \frac{\sumK \gamma_i}{n}\Big)\,.
\end{align*}
\end{proof}

\begin{lemma}\label{lem:multitobino}
    If $N_1, \ldots, N_K \sim M(n, \p(1), \ldots, \p(K))$ and $\tilde{N}_i \sim B(n, \min\{1, 3 \p(i)\})$, then for any $R \subset [K]$ such that $\sum_{i \not \in R} \pie \geq 1/3$, any $b_1, \ldots, b_K \geq 0$, and any $z > 0$
    \begin{align*}
        \Pp(\sum_{i \in R} \id\{N_i \geq b_i\} \geq z) \leq \Pp(\sum_{i \in R} \id\{\tilde N_i \geq b_i\} \geq z)\,.
    \end{align*}
\end{lemma}
\begin{proof}
For simplicity we only show the proof for $R = [K-1]$ and assume $p(K) \geq \frac13$. The general proof follows from the proof for this case. For any $i'$ and $R' = [K-1]\setminus i'$ a basic property of the multinomial distribution is that (see, e.g. \citep[Chapter~4]{roussas2003introduction}) 
\begin{align*}
    N_{i'}|\{N_i : i \in R'\} \sim B(n - \sum_{i \in R'}N_i, \frac{p(i')}{\sum_{i \not \in R'} p(j)}) \,.
\end{align*}
Thus, since $\sum_{i \not \in R'} p(j) > p(K) \geq \frac{1}{3}$, we have that for any $b \geq 0$
\begin{align}\label{eq:mtob}
    \Pp(N_1 \geq b|N_2, \ldots, N_{K-1}) \leq \Pp(\tilde{N}_1 \geq b)\,.
\end{align}
Denote by $Y_i = \id\{N_i \geq b_i\}$, by $\tilde{Y}= \id\{\tilde{N_i} \geq b_i\}$, and by
\begin{align*}
    F = \{x_2,\ldots, x_{K-1}: \sum_{i=2}^{K-1} \id\{x_i > b_i\} = z - 1 \}\,.
\end{align*}
We have 
\begin{align*}
    & \Pp(\sum_{i=1}^{K-1} Y_i \geq z) \\
    & = \Pp(Y_1 \geq z - \sum_{i = 2}^{K-1} Y_i) \\
    & = \sum_{j=1}^{K-2} \Pp (\sum_{i = 2}^{K-1} Y_i = j) \Pp(Y_1 \geq z - j|\sum_{i = 2}^{K-1} Y_i = j) \\
    & \overset{a}{=} \Pp (\sum_{i = 2}^{K-1} Y_i > z-1) + \Pp (\sum_{i = 2}^{K-1} Y_i = z-1) \Pp(Y_1 = 1|\sum_{i = 2}^{K-1} Y_i = z-1) \\
    & = \Pp (\sum_{i = 2}^{K-1} Y_i > z-1) + \Pp(Y_1 = 1 \bigcap \left(\sum_{i = 2}^{K-1} Y_i = z-1\right))\\
    & = \Pp(\sum_{i = 2}^{K-1} Y_i > z-1) + \sum_{(x_2, \ldots, x_{K-1}) \in F} \Pp(Y_1 = 1 \bigcap (N_2 = x_2, \ldots, N_{K-1} = x_{K-1})) \\
    & = \Pp(\sum_{i = 2}^K Y_i > z-1) + \sum_{(x_2, \ldots, x_{K-1}) \in F} \Pp(N_1 \geq b_1 \bigcap (N_2 = x_2, \ldots, N_{K-1} = x_{K-1})) \\
    & \leq  \Pp(\sum_{i = 2}^K Y_i > z-1) + \sum_{(x_2, \ldots, x_{K-1}) \in F} \Pp(\tilde N_1 \geq b_1 \bigcap (N_2 = x_2, \ldots, N_{K-1} = x_{K-1}))  \\
    & = \Pp(\tilde{Y}_1 + \sum_{i = 2}^{K-1} Y_i \geq z)\,,
\end{align*}
where $a$ follows from the fact that $Y_1$ can only take values $0$ or $1$ and the inequality is due to \eqref{eq:mtob}.
Repeating the above argument shows that $\Pp(\sum_{i=1}^{K-1} Y_i \geq z) \leq \Pp(\sum_{i=1}^{K-1} \tilde Y_i \geq z)$.
\end{proof}

\section{Additional Proofs for Section~\ref{sec:upperbounds}}\label{app:upperbounds}

\lemratioOTB*

\begin{proof}
    We start with
    \begin{align*}
        \frac{\pie}{\ppi} = \frac{n + \sumK \gamma_i}{n} \frac{n\pie}{n \pie + \gamma_i} \leq 1 + \frac{\sumK \gamma_i}{n}\,.
    \end{align*}
    By the empirical Bernstein inequality \citep[Theorem 4]{maurer2009empirical} and a union bound we have that for all $i \in [K]$, with probability at least $1 - 2 K \delta'$ 
    \begin{align}\label{eq:empbern}
        |\nin - \pie n| & \leq \sqrt{2\nin \ln(2/\delta)} + 5\ln(2/\delta') \notag \\
        |\ninh - \pie n| & \leq \sqrt{2\ninh \ln(2/\delta)} + 5\ln(2/\delta')\,.
    \end{align}
    On this event, by the AM-GM inequality, for any fixed $\eta > 0$, we have that
    \begin{align}\label{eq:empbern2}
        (1-\eta) \nin - \left(\frac{2}{\eta} + 5\right) \ln(2/\delta') \leq  2 n \pie \leq (1+\eta) \nin + \left(\frac{2}{\eta} + 5\right) \ln(2/\delta') \notag\\
        (1-\eta) \ninh - \left(\frac{2}{\eta} + 5\right) \ln(2/\delta') \leq  2 n \pie \leq (1+\eta) \ninh + \left(\frac{2}{\eta} + 5\right) \ln(2/\delta')\,.
    \end{align}
    On the event that \eqref{eq:empbern} holds, we thus have that
    \begin{align*}
        \frac{\pie}{\pti} & \leq \frac{n + \sumK \gamma_i}{n} \frac{n \pie}{\ninh + \gamma_i} \\
        & \leq \frac{n + \sumK \gamma_i}{n} \frac{\half (1+\eta) \ninh + \left(\frac{2}{\eta} + 5\right) \ln(2/\delta')}{\ninh + \gamma_i} \\
        & \leq \left(1 + \frac{\sumK \gamma_i}{n} \right)\left(\frac{1+\eta}{2} + \frac{\left(\frac{2}{\eta} + 5\right) \ln(2/\delta')}{\max\{\ninh, \gamma_i\}}\right)\,.
    \end{align*}
    Likewise, on the event that \eqref{eq:empbern} holds, we have that
    \begin{align*}
        \frac{\ppi}{\pti} & \leq \frac{n \pie + \gamma_i}{\ninh + \gamma_i} \\
        & \leq 1 + \frac{n\pie}{\ninh + \gamma_i} \\
        & \leq 1 + \left(1 + \frac{\sumK \gamma_i}{n} \right)\left(\frac{1+\eta}{2} + \frac{\left(\frac{2}{\eta} + 5\right) \ln(2/\delta')}{\max\{\ninh, \gamma_i\}}\right)\,.
    \end{align*}
    And once more, on the event that \eqref{eq:empbern} holds, we have that
    \begin{align*}
        \frac{\pti}{\ppi} & = \frac{n + \sumK \gamma_i}{t-1 + \sumK \gamma_i} \frac{\nitm + \gamma_i}{n \pie + \gamma_i} \\
        & \leq 3\frac{\nin + \gamma_i}{n \pie + \gamma_i} \\
        & \leq 3 \left(1 + \frac{\nin}{\max\{\gamma_i, (1-\eta) \ninh - \left(\frac{2}{\eta} + 5\right) \ln(2/\delta')\}}\right)
    \end{align*}
    Furthermore, we have
    \begin{align}\label{eq:nintoninh}
        \nin & \leq \frac{1}{1-\eta} \left(2n \pie + \left(\frac{2}{\eta} + 5\right) \ln(2/\delta') \right) \notag \\
        & \leq \frac{1}{1-\eta} \left(2(1+\eta) \ninh + 3\left(\frac{2}{\eta} + 5\right) \ln(2/\delta')\right) \,,
    \end{align}
    and thus
    \begin{align*}
        \frac{\pti}{\ppi} & \leq 3 \left(1 + \frac{\frac{1}{1-\eta} \left(2(1+\eta) \ninh + 3\left(\frac{2}{\eta} + 5\right) \ln(2/\delta')\right)}{\max\{\gamma_i, (1-\eta) \ninh - \left(\frac{2}{\eta} + 5\right) \ln(2/\delta')\}}\right) \,.
    \end{align*}
    Finally, by equation~\ref{eq:nintoninh} we also have 
    \begin{align*}
        \frac{\nin + \gamma_i}{\ninh + \gamma_i} & \leq \frac{2(1+\eta)}{1-\eta} + \frac{\left(\frac{2}{\eta} + 5\right) \ln(2/\delta')}{(1- \eta)(\ninh + \gamma_i)}
    \end{align*}
    Setting $\delta' = \frac{\delta}{2K}$ and $\eta = \half$ completes the proof. 
\end{proof}

\lemdivergenceequivalence*
\begin{proof}
Let $R$ be a self-concordant function, let $r = \sqrt{(x - y)(R''(x))(x - y)}$, and let $\rho(u) = -\log(1-u) - u$.
By \citep[2.4]{nemirovski96},
\begin{align*}
R(y) \geq R(x) + \langle\nabla R(x), y - x\rangle + \rho(-r) 
\geq R(x) + \langle \nabla R(x), y - x\rangle + \frac{\min(r, r^2)}{4}\,.
\end{align*}
Since $-\log(x)$ is self-concordant, we have
\begin{align*}
    \pie (\log(p(i)) - \log(q(i))) & \geq \pie \Big(\frac{-1}{p(i)}(q(i) - p(i)) + \frac{1}{4}\min \{r, r^2\}\Big)\,,
\end{align*}
where 
\begin{align*}
    r^2 & = \frac{(\pie - \bar{p}_{n}(i))^2}{\pie^2} \leq \frac{1}{\pie^2}\left(\frac{7}{32}\bar{p}_n(i)\right)^2 \leq \frac{1}{\pie^2}\left(\frac{14}{32}\pie\right)^2 \leq \frac{1}{4}\,.
\end{align*}
So, we have
\begin{align*}
    \KL(p \| q) & = \sumK p(i) (\log(p(i)) - \log(q(i))) \\
    & \geq \sumK p(i) \Big(\frac{-1}{p(i)}(q(i) - p(i)) + \frac{1}{4}\min \{r, r^2\}\Big) \\
    & = \sumK \Big((q(i) - p(i)) + \frac{1}{4}\frac{(p(i) - q(i))^2}{p(i)}\Big) \\
    & = \frac{1}{4}\sumK\frac{(p(i) - q(i))^2}{p(i)} \\
    & = \frac{1}{4} \chi^2 (q,p)\\
    & \geq \frac{1}{6}\sumK\frac{(p(i) - q(i))^2}{q(i)} \\
    & = \frac{1}{6} \chi^2 (p,q)\,.
\end{align*}
By Lemma~4 in \citep{yang1998asymptotic} we have that 
\begin{align*}
    & \KL(p \| q) \leq (2 + \log(3/2)) H^2(p , q) \leq (2 + \log(3/2)) \KL(q\|p) \\
    & \leq (2 + \log(3/2)) \chi^2(q,p) \leq (4 + 2\log(3/2)) \chi^2(p,q) \,,
\end{align*}
where $H^2(p , q) \leq \KL(q\|p) \leq \chi^2(q,p)$ can be found in \citep{gibbs2002choosing}.
Na\"ively bounding $\log(3/2) < 1/2$ and simplifying completes the proof. 
\end{proof}

\thdivergences*
\begin{proof}
By Bennet's inequality \citep[Theorem 4]{maurer2009empirical} and a union bound we have that for all $i \in [K]$, with probability at least $1 - \delta$ 
\begin{align*}
    |\nin - \pie n| \leq \sqrt{2n\pie \ln(4K/\delta)} + 5\ln(4K/\delta) \leq \frac{7}{32}n \pie \,,
\end{align*}
where the last inequality follows from the assumption that $|\Jtcal| = 0$. 
Thus, on this event, for all $i \in [K]$
\begin{align*}
    \bar{p}_n(i) \in \left[\frac{25}{32} \pie, \frac{39}{32} \pie)\right] \subset \left[\frac{1}{2} \pie, \frac{3}{2} \pie\right]\,.
\end{align*}
At this point, an application of Lemma~\ref{lem:divergenceequivalence} allows us to complete the proof of the second part of the statement. For $\Pp\left(\KL(\p \| \bar{p}_{n}) \geq C \frac{K + \log(1/\delta)}{n}\right) \leq \delta$ we can apply Lemma~\ref{lem:divergenceequivalence} to obtain $\KL(\p \| \bar{p}_{n}) \leq \KL(\bar{p}_{n}\|\p )$, after which we can apply Lemma~\ref{lem:kl_simple_mean} to complete the proof.

\end{proof}

\begin{lemma}\label{lem:helltohell}
    We have that
    \begin{align*}
        H^2(\p,p_{n+1}) - H^2(\p,\bar p_n) \leq \frac{\sumK
        \gamma_i}{2n}.
    \end{align*}
\end{lemma}
\begin{proof}
    By definition of the Hellinger distance we have that
\begin{align*}
    H^2(\p,p_{n+1}) - H^2(\p,\bar p_n) & = \frac{1}{2}\sumK \sqrt{\pie}\big(\sqrt{\bar p_n(i)} - \sqrt{{p}_{n+1}(i)}\big) \\
    & = \frac{1}{2}\sumK \sqrt{\pie}\Big(\sqrt{\bar p_n(i)} - \sqrt{\bar{p}_n(i) \frac{n}{n+\sumK \gamma_i} + \frac{\gamma_i}{n+\sumK \gamma_i}}\Big) \\
    & \leq \frac{1}{2}\sumK \sqrt{\pie\bar{p}_n(i)}\Big(1 - \sqrt{\frac{n}{n+\sumK \gamma_i}}\Big) \\
    & = \frac{1}{2}\sumK \sqrt{\pie\bar{p}_n(i)}\frac{n+\sumK \gamma_i - \sqrt{n+\sumK \gamma_i}\sqrt{n}}{n+\sumK \gamma_i} \\
    & \leq \frac{\sumK \gamma_i}{2(n+\sumK \gamma_i)} \sumK \sqrt{\pie\bar{p}_n(i)}  \\
    & = \frac{\sumK \gamma_i}{2(n+\sumK \gamma_i)} (1 - H^2(p,\bar{p}_n(i))) \\
    & \leq \frac{\sumK \gamma_i}{2n}\,.
\end{align*}
\end{proof}

\begin{restatable}{relemma}{lemKLtoReg}\label{lem:KLtoReg}
    On event $\mathcal{Z}$, with probability at least $1-\delta$ we have that
    \begin{align*}
        \frac{n}{2} \KL(\p\|\potb) \leq 2\regret_T + 2 \sumK \gamma_i + (4 + 2 \log(\beta))\log(1/\delta)\,,
    \end{align*}
    where $\regret_T = \sumnhp \left( -\ln{p_t(X_t)} - (- \ln{\pp(X_t)} )\right)$.
\end{restatable}
\begin{proof}
We start with an application of Jensen's inequality 
\begin{align*}\label{eq:KL-lowerbound_use}
	\KL(\p\|\potb) & = \sumK \pie \log\left(\frac{\pie}{\potbi}\right) \\
    & = \sumK \pie \bigg(\log\left(\frac{\ppi}{\potbi}\right) + \log\left(\frac{\pie}{\ppi}\right)\bigg) \\
    & \leq \sumK \pie \bigg(\log\left(\frac{\ppi}{\potbi}\right) + \log\left(1 + \frac{\sumK \gamma_i}{n}\right)\bigg) \tag{\textnormal{By \eqref{eq:eventOTB}}}\\
    & \leq \sumK \pie\log\left(\frac{\ppi}{\potbi}\right) + \frac{\sumK \gamma_i}{n} \tag{\textnormal{$\log(1 + x) \leq x$ for $x > -1$}}\\
    & \leq \frac{2}{n} \sumnhp \sumK \pie\log\left(\frac{\ppi}{\pti}\right) + \frac{\sumK \gamma_i}{n} \tag{Jensen's inequality}
\end{align*}
We need the following concentration inequality for martingales whose proof can be found in \cite[Theorem 1]{beygelzimer2011contextual}.
\begin{lemma}[A version of Freedman's inequality]\label{lem:bernie}
	Let $X_1, \ldots, X_T$ be a martingale
	difference sequence adapted to a filtration $(\mathcal{F}_i)_{i \le T}$. That is, in particular, $\E_{t-1}[X_t] = 0$.
	Suppose that $|X_t| \leq R$ almost surely. Then for any $\delta \in (0,1), \lambda \in [0, 1/R]$, with probability at least $1-\delta$, it holds that
	\begin{equation*}
		\label{eq:firstfreedmanineq} 
		\sumT X_t \leq \lambda(e-2)\sumT \E_{t-1}[X_t^2] + \frac{\ln(1/\delta)}{\lambda}~.
	\end{equation*}
\end{lemma}

By equation~\ref{eq:eventOTB}, we know that
\begin{align}\label{eq:boundedevent}
    |\log\left(\frac{\ppi}{\pti}\right)| \leq \log\left(\left(1 + \frac{\sumK \gamma_i}{n} \right)3 \left(1 + \frac{2 \left(3 \ninh + 27\ln(4K/\delta)\right)}{\max\{\gamma_i, \half \ninh - 9 \ln(4K/\delta)\}}\right)\right) = \log(\beta_i)
\end{align}
Denote $\lambda^+ = \log(\max_i\beta_i)$. By Lemma~\ref{lem:bernie}, with probability at least $1 - \delta$, for any $\lambda \in [0, \frac{1}{\lambda^+}]$
\begin{align*}
	& \sumnhp \Big(\sumK \pie\log (\frac{\ppi}{\pti}) - \log(\frac{\pp(X_t)}{\pt(X_t)})\Big) \\
    & \leq \lambda \sumnhp \sumK \pie(\log\left(\frac{\ppi}{\pti}\right)^2 + \frac{\log(1/\delta)}{\lambda} \tag{$\E[(X - \E[X])^2] \leq \E[X^2]$}\, ,
\end{align*}
where we used that $\E[(X - \E[X])^2] \leq \E[X^2]$. 
On event \eqref{eq:boundedevent}, by Corollary~\ref{cor:applyYB} we have that 
\begin{align*}
    & \lambda \sumnhp \sumK \pie(\log\left(\frac{\ppi}{\pti}\right)^2 \\
    & \leq \lambda (2 + \log(\beta)) \sumnhp \sumK \pie\log\left(\frac{\ppi}{\pti}\right) + \frac{\lambda}{2} (2 + \log(\beta)) \sumK \gamma_i\,.
\end{align*}
Setting $\lambda = \frac{1}{4 + 2 \log(\beta)}$ we can thus conclude that with probability at least $1 - \delta$
\begin{align*}
    & \sumnhp \sumK \pie\log\left(\frac{\ppi}{\pti}\right) \\
    & \leq 2\sumnhp \sumK \left( -\ln{p_t(X_t)} - (- \ln{\pp(X_t)} )\right) \\
    & \quad + \sumK \gamma_i + (4 + 2 \log(\beta))\log(1/\delta) \,,
\end{align*}
and consequently 
\begin{align*}
    & \KL(\p\|\potb) \\
    & \leq \frac{2}{n}\Big(2\sumnhp \sumK \left( -\ln{p_t(X_t)} - (- \ln{\pp(X_t)} )\right) \\
    & \quad + 2 \sumK \gamma_i + (4 + 2 \log(\beta))\log(1/\delta)\Big)\,.
\end{align*}

\end{proof}

\begin{restatable}{relemma}{lemregretbound}\label{lem:regretbound}
    We have that
    \begin{align*}
    \regret_T & \leq \sumK \gamma_i \log(\pnhpi/\ppi) +  2\sumK \gamma_i + \frac{n}{2} \KL(\bar{p}_{n/2+1}\|\p)
	 + \sumK \log\left(\frac{\nin+ \gamma_i}{\ninh+\gamma_i}\right)\,.
    \end{align*}
\end{restatable}
\begin{proof}
Recall that
\begin{align*}
    \mathcal{R}_T = \sumnhp \left( -\ln{p_t(X_t)} - (- \ln{\pp(X_t)} )
    )\right).
\end{align*}
We can write the computation of $\ptpi$ as the action of an FTRL algorithm:
\begin{align*}
    \ptpi = \argmin_{p \in \Delta^K} \sum_{t' = n/2 + 1}^t \sumK \left(-\id[X_{t'} = i]\ln(p(i))\right) + R(p)\,,
\end{align*}
where $R(p)$ is the regularizer and is defined as 
\begin{align*}
    R(p) = \sumK \left(- \gamma_i \ln(p(i))  + \sum_{t=1}^{n/2} - \id[X_{t} = i]\ln(p(i))\right)\,.
\end{align*}
Denote by $\phi_t(p) = \sum_{t'=n/2+1}^{t} \sumK - \id[X_{t'}=i]\ln(p(i))+ R(p)$ the FTRL potential. 
By Lemma~7.1 in \citep{orabona2019modern}, we have that
\begin{align*}
    &\mathcal{R}_T = \phi_T(p_{n+1}) -  R(p_{n/2+1}) \\
    & \quad - \sumnhp \sumK  (-\ln (\pp(X_t))) + \sumnhp (\phi_t(p_t) - \phi_t(p_{t+1})) \\
    & = \phi_T(p_{n+1}) - \phi_T(\pp) -  R(p_{n/2+1}) + R(\pp) \\
    & \quad + \sumnhp (\phi_t(p_t) - \phi_t(p_{t+1}))\,.
\end{align*}
Now, since $\phi_T(p_{n+1}) - \phi_T(\pp)\leq 0$ and $\phi_{t-1}(\pt) - \phi_{t-1}(\ptp)\leq 0$ we have that 
\begin{align*}
    &\mathcal{R}_T \leq R(\pp) - R(p_{n/2+1}) + \sum_{t = n/2+1}^{n} \ln \left(\frac{p_{t+1}(X_t)}{p_t(X_t)}\right)\\
    & = \sumK \gamma_i \log(\pnhpi/\ppi) +  \sum_{t = 1}^{n/2} \log(\pnhp(X_t)/\pp(X_t)) \notag \\
	& \quad + \sum_{t = n/2+1}^{n} \sumK \id[X_t = i] \log(\frac{(n_{t,i}+ \gamma_i)(t - 1 + \sumK \gamma_i)}{(n_{t-1,i}+\gamma_i)(t + \sumK \gamma_i)}) \notag\\
    & \leq \sumK \gamma_i \log(\pnhpi/\ppi) +  \sum_{t = 1}^{n/2} \log(\pnhp(X_t)/\pp(X_t)) \notag \\
	& \quad + \sum_{t = n/2+1}^{n} \sumK \id[X_t = i] \log(\frac{n_{t,i}+ \gamma_i}{n_{t-1,i}+\gamma_i}) \tag{$\frac{t - 1 + \sumK \gamma_i}{t + \sumK \gamma_i} \leq 1$}\\
    & = \sumK \gamma_i \log(\pnhpi/\ppi) +  \sum_{t = 1}^{n/2}
    \log(\pnhp(X_t)/\pp(X_t)) 
	 + \sumK \log\left(\frac{\nin+
         \gamma_i}{\ninh+\gamma_i}\right)\,,\notag
\end{align*}
where the last equality follows from a telescoping sum.
At this point, we can use that $\sum_{t=1}^{n/2}\id[X_t=i] = \frac{n}{2} \bar{p}_{n/2+1}(i)$ to see that  
\begin{align}
    & \sum_{t=1}^{n/2} \log(\pnhp(X_t)/\pp(X_t)) \notag \\
    & =  \sumK \sum_{t=1}^{n/2} \id[X_t=i] \pa{\log\pa{\frac{\pnhpi}{\bar{p}_{n/2}(i)}}+\log\pa{\frac{\bar{p}_{n/2}(i)}{\pie}} +\log\pa{\frac{\pie}{\ppi}}}\notag\\
    & =  \sumK \sum_{t=1}^{n/2} \id[X_t=i] \pa{\log\pa{\frac{\pnhpi}{\bar{p}_{n/2}(i)}} +\log\pa{\frac{\pie}{\ppi}}} +  \frac{n}{2} \KL(\bar{p}_{n/2+1}\|\p)\notag\\
    & \leq  \sumK \sum_{t=1}^{n/2} \id[X_t=i] \pa{\log\pa{\frac{\pnhpi}{\bar{p}_{n/2}(i)}} +\frac{\gamma_i}{n}} +  \frac{n}{2} \KL(\bar{p}_{n/2+1}\|\p) \tag{$\frac{\pie}{\ppi} \leq 1 + \frac{\sumK \gamma_i}{n}$ and $\log(1+x)\le x$ for $x>0$}\\
    & = \sumK \sum_{t=1}^{n/2} \id[X_t=i] \pa{ \log\pa{\frac{(\ninh + \gamma_i)n/2}{\ninh (n/2 + \sumK \gamma_i)}}+\frac{\sumK \gamma_i}{n}} +  \frac{n}{2} \KL(\bar{p}_{n/2+1}\|\p) \tag{By definition}\\
    & \leq \sumK \sum_{t=1}^{n/2} \id[X_t=i] \pa{\frac{\sumK \gamma_i}{\ninh} +\frac{\sumK \gamma_i}{n}} +  \frac{n}{2} \KL(\bar{p}_{n/2+1}\|\p)\tag{Since $\frac{n/2}{n/2 + \sumK \gamma_i} \leq 1$}\\
    & \leq 2\sumK \gamma_i + \frac{n}{2} \KL(\bar{p}_{n/2+1}\|\p)\,, \notag 
\end{align}
which combined with the above completes the proof. 
\end{proof}

\begin{restatable}{relemma}{lemKLsimplemean}\label{lem:kl_simple_mean}
    With probability at least $1 - \delta$
    \begin{align*}
        \KL(\bar{p}_{n+1}\|\p) \leq \E[\KL(\bar{p}_{n+1}\|\p)] +
        \frac{6K + 6\ln(1/\delta)}{n} \leq \frac{7K +
        6\ln(1/\delta)}{n}.
    \end{align*}
\end{restatable}
\begin{proof}
    The first inequality can be found as Corollary~1.7 in
    \citep{agrawal2022finite}. The second inequality uses that
    $\E[\KL(\bar{p}_{n+1}\|\p)] \leq \frac{K-1}{n}$
    \citep[Section~4]{paninski2003estimation}. 
\end{proof}

\begin{restatable}{relemma}{lemcaseJone}\label{lem:caseJ1}
    Suppose that $\frac{\log(K/\delta)}{J} > 1$. Then on event $\mathcal{Z}$ we have that
    \begin{align*}
        & \sumK \gamma_i \big(\log(\pnhpi/\ppi) + 3\big) + \sumK \log\left(\frac{\nin+ \gamma_i}{\ninh+\gamma_i}\right) + (2 + \ln(\beta))\ln(1/\delta)\\
        &\leq 7\log(K/\delta) +  20 K +
        2\log(1/\delta)\log\left(400J\right).
    \end{align*}
\end{restatable}
\begin{proof}
    On event $\mathcal{Z}$ we have that
    \begin{align*}
        & \sumK \gamma_i \log(\pnhpi/\ppi) + \sumK \log\left(\frac{\nin+ \gamma_i}{\ninh+\gamma_i}\right) \\
        & \leq \sumK \gamma_i \log\left(3 + \frac{6 \left(3 \ninh + 27\ln(4K/\delta)\right)}{\max\{\gamma_i, \half \ninh - 9 \ln(4K/\delta)\}}\right) + \sumK \log\left(6 + \frac{18 \log(K/\delta)}{\ninh+\gamma_i}\right)\,. \\
        & \leq 2 K \log(100) + \log(K/\delta) + J\log(J) \\
        & \leq 2 K \log(100) + \log(K/\delta) + \log(K/\delta)\log\left(J\right) \\
        & \leq 2 K \log(100) + \log(1/\delta) + \log(1/\delta)\log\left( J\right) + 2\log(K)^2 \\
        & \leq 20 K + \log(1/\delta) + \log(1/\delta)\log\left(400J\right)
    \end{align*}
    where we used that $J < \log(K/\delta)$ and $\log(K)^2 \leq K$. By definition of $\gamma_i$ we can also bound
    \begin{align*}
        \sumK \gamma_i = \log(K/\delta)\,.
    \end{align*}
    Finally, since $J < \log(K/\delta)$ we can also see that $\beta \leq 400 J$. 
\end{proof}

\begin{restatable}{relemma}{lemcaseJtwo}\label{lem:caseJ2}
     Suppose that $\frac{\log(K/\delta)}{J} \leq 1$. Then on event $\mathcal{Z}$ we have that with probability at least $1 - 3\delta$
    \begin{align*}
        & \sumK \gamma_i \big(\log(\pnhpi/\ppi) + 3\big) + \sumK \log\left(\frac{\nin+ \gamma_i}{\ninh+\gamma_i}\right) + (2 + \ln(\beta))\ln(1/\delta)\\
        &\leq 50\log(1/\delta) \log(800J) + 2\Jt\log\left(24\left(\log\left(\frac{\Jt}{\log(1/\delta)}\right) \vee 1\right)\right) + 10 K \log(100)\,.
    \end{align*}
\end{restatable}
\begin{proof}
Since $\frac{\log(K/\delta)}{J} \leq 1$ and we have $\gamma_i = 1$ or $\gamma_i = 0$ and by definition $\sumK \gamma_i = J$. Furthermore, we can see that
\begin{align*}
    \beta & \leq 800\log(K/\delta) \leq 800J\,,
\end{align*}
Recall that $\Jtcal = \{i: n\pie < 32\log(K/\delta)\}$ and $\Jt = \max\{3, |\Jtcal|\}$. On event $\mathcal{Z}$ we have that $\frac{\nin + \gamma_i}{n\pie + \gamma_i} \leq \beta_i$ for all $i$. For $i \not \in \Jcal$ we have that $\frac{\nin}{\ninh} \leq \beta_i \leq 3 + \frac{6 \left(3 \ninh + 27\ln(4K/\delta)\right)}{\max\{\gamma_i, \half \ninh - 9 \ln(4K/\delta)\}} \leq 100$ and thus 
\begin{align*}
    & \sumK \gamma_i \log(p_{n/2+1}/\ppi) + \sumK \log\left(\frac{\nin+ \gamma_i}{\ninh+\gamma_i}\right) \\
    & \leq \sum_{i \in \Jcal} \log\left(\frac{\nin + 1}{n\pie + 1}\right) + \sum_{i \in [K]\setminus \Jcal} \log\left(\frac{\nin}{\ninh} \right) \\
    & \leq \sum_{i \in \Jcal} \log\left(\frac{\nin + 1}{n\pie + 1}\right) + K\log(100)\,.
\end{align*}
By a union bound and Bennet's inequality, for all $i \in [K]$, with probability at least $1 - \delta$, we have that 
\begin{align*}
    |\nin - n\pie| \leq \sqrt{2n\pie \log(K/\delta)} + \log(K/\delta)/3\,.
\end{align*}
On this event, for $i \not \in \Jtcal$, we have that
\begin{align*}
    |\nin - n\pie| \leq 2 n\pie\,,
\end{align*}
which implies that $\nin \leq 3n\pie$. On the same event, for $i \in \Jtcal$ we have that $\nin \leq 9 \log(K/\delta)$. Therefore, we have that
\begin{align*}
    & \sum_{i \in \Jcal} \log\left(\frac{\nin + 1}{n\pie + 1}\right) + K\log(100) \\
    & \leq \sum_{i \in \Jtcal} \log\left(\frac{\nin + 1}{n\pie + 1}\right) + 2K\log(100)\,,
\end{align*}
where we used that for $i \not \in \Jcal$ $\nin \geq 32\log(K/\delta)$.
Let $\p_{\min} = \min_{i \in [K]} \pie$. Suppose that $\sum_{i \in [K] \setminus \Jtcal} \pie > \frac{1}{3}$. Let $\tilde{X}_i \sim B(n, \min\{1, 3\pie\})$. Let $x \geq 1$ and $z \geq 2 \Jt \exp(-\tfrac{1}{4} {x(3n\p_{\min} + 1)})$. By Lemma~\ref{lem:multitobino} we have that 
\begin{align*}
    \Pp(\sum_{i \in \Jtcal} \id\{\nin > 3n\pie + x(3 n\pie + 1)\} \geq z) \leq \Pp(\sum_{i \in \Jtcal} \id\{\tilde{X}_i > 3n\pie + x(3 n\pie + 1)\} \geq z)  
\end{align*}
By Bernstein's inequality, we have that 
\begin{align*}
    \Pp(\tilde{X}_i > 3n\pie + x(3 n\pie + 1)) & \leq \exp\left(-\frac{\half (x(3 n\pie + 1))^2}{3n\pie + \tfrac{1}{3}x(3 n\pie + 1)}\right) \\
    & \leq  \exp\left(- \tfrac{1}{4} {x(3np_{\min} + 1)}\right)
\end{align*}
Let $Z \sim B(\Jt, \exp(-\tfrac{1}{4} {x(3np_{\min} + 1)}))$. By Bernstein's inequality, we have that
\begin{align*}
    & \Pp(\sum_{i \in \Jtcal} \id\{\tilde{X}_i > 3n\pie + x(3 n\pie + 1)\} \geq z) \\
    & = \Pp (Z \geq z) \\
    & = \Pp(Z - \E[Z] \geq z - \E[Z]) \\
    & \leq \Pp(Z - \E[Z] \geq \half z) \\
    & \leq \exp\left(-\frac{1}{8}\frac{z^2}{\Jt\exp(- \tfrac{1}{4} {x(3np_{\min} + 1)}) + 1/6 z}\right) \\
    & \leq \exp\left(-\frac{z}{24}\right)
\end{align*}
Let $\mathcal{Z} = \{i \in \Jtcal: \nin > 3n\pie + x(3 n\pie + 1)\}$.  With probability at least $1 - \exp(-z/24)$ we have that $|\mathcal{Z}| \leq z$ and thus 
\begin{align*}
    \leq \sum_{i \in \Jtcal} \log\left(\frac{\nin + 1}{n\pie + 1}\right) + 2K\log(100)
    & \leq \sum_{i \in \mathcal{Z}} \log\left(\frac{\nin + 1}{n\pie + 1}\right) + (\Jt - |\mathcal{Z}|)\log(6x) + 2K \log(100) \\
    & \leq z \log(400(1 + \log(K/\delta)))+ (\Jt- |\mathcal{Z}|)\log(6x) + K \log(100)\,.
\end{align*}

Setting $z = \max\{24 \log(1/\delta), 2\Jt\exp(-\tfrac{x}{4}(3n \p_{\min} + 1))\}$ and $x = 4\left(\log\left(\frac{\Jt}{\log(1/\delta)}\right)\vee 1\right)$ we find that with probability at least $1-2\delta$
\begin{align*}
    & \sumK \gamma_i \log(p_{n/2+1}/\ppi) + \sumK \log\left(\frac{\nin+ \gamma_i}{\ninh+\gamma_i}\right) \\
    & \leq 24\log(1/\delta) \log(400(1 + \log(K/\delta)))+ \Jt\log\left(24\log\left(\log\left(\frac{\Jt}{\log(1/\delta)}\right) \vee 1\right)\right) + K \log(100)\,.
\end{align*}
Now, suppose that $\sum_{i \in [K] \setminus \Jtcal} \pie < \frac{1}{3}$. In that case we can always construct a two sets $R_1 \subset \Jtcal$ and $R_2 \subset \Jtcal$ such that $R_1 \bigcap R_2 = \Jtcal$, $\sum_{i \not \in R_1} \pie \geq \frac{1}{3}$ and  $\sum_{i \not \in R_2} \pie \geq \frac{1}{3}$. With these two sets we can repeat the analysis for the first case, where we apply Lemma~\ref{lem:multitobino} to $R_1$ and $R_2$ separately, except now we need a union bound to combine the analyses. Ultimately, we find that, with probability at least $1 - 3\delta$
\begin{align*}
    & \sumK \gamma_i \log(p_{n/2+1}/\ppi) + \sumK \log\left(\frac{\nin+ \gamma_i}{\ninh+\gamma_i}\right) \\
    & \leq 48\log(1/\delta) \log(400(1 + \log(K/\delta))) + 2\Jt\log\left(24\left(\log\left(\frac{\Jt}{\log(1/\delta)}\right) \vee 1\right)\right) + 4 K \log(100) \,.
\end{align*}
At this point we can use $\ln(K/\delta) \leq J$ to complete the proof. 
\end{proof}

\section{Additional Proofs for Section~\ref{sec:lowerbounds}}\label{app:lowerbounds}

\thmDeltaDependentlb*
\begin{proof}
  We construct a distribution $\p$ that is hard for $\phat$ as follows:
  for $i' = \argmin_{i\in\{2,\ldots,K\}}\phat^0(i)$ let
  \[
    \p(i) = \begin{cases}
      1-\frac{\alpha}{n} & \text{if $i = 1$}\\
      \frac{\alpha}{n} & \text{if $i = i'$}\\
      0                & \text{otherwise,}
    \end{cases}
  \]
  where $0 < \alpha \leq n/2$ will be chosen later.
  For this choice of $\p$, using that $\ln(1-x) \ge -x-x^2$ for
  $0<x<\frac{1}{2}$, we have that
\begin{align*}
    \Pp(n_1=1) &= \pa{1-\frac{\alpha}{n}}^n = \exp \pa{n\ln\pa{1-\frac{\alpha}{n}}}\ge \exp\pa{ -\alpha - \frac{\alpha^2}{n}}
     >\exp \pa{-\frac{3\alpha}{2}},
\end{align*}
where the last inequality follows from $n> 2\alpha$, and the event
$n_1=n$ is equivalent to $X_1 = \cdots = X_n = 1$. By choosing $\alpha =
\frac{2}{3}\log(1/\delta)$, we find that
\[
  \Pp(n_1=n) > \delta.
\]

On the event that $n_1=n$, the KL divergence of $\p$ from $\phat$ is 
\begin{align*}
    \KL(\p\|\phat) & = \frac{\alpha}{n}\ln\left(\frac{\alpha}{n \phat^0(i')}\right) + \pa{1 - \frac{\alpha}{n}}\ln\left(\frac{1 - \frac{\alpha}{n}}{1 - \sum_{i = 2}^{K}\phat^0(i)}\right) \\
    & \geq \frac{\alpha}{n}\ln\left(\frac{\alpha}{n \phat^0(i')}\right) + \sum_{i = 2}^{K}\phat^0(i) - \frac{\alpha}{n}\,,
\end{align*}
where the last inequality follows by $\log(x) \ge 1-1/x$ for $x>0$.
By the definition of $i'$, we get that $\phat^0(i') \le \sum_{i=2}^{K}\frac{\phat^0(i)}{(K-1)}$, which gives us
\begin{align*}
    \KL(\p\|\phat) & \geq \frac{\alpha}{n}\ln\left(\frac{\alpha}{n \phat^0(i')}\right) + \sum_{i = 2}^{K}\phat^0(i) - \frac{\alpha}{n}\\
    & \ge \frac{\alpha}{n}\ln\left(\frac{\alpha (K-1)}{n\sum_{i=2}^{K}\phat^0(i)}\right) + \sum_{i = 2}^{K}\phat^0(i) - \frac{\alpha}{n}.
\end{align*}
Recalling that $\alpha = \frac{2}{3}\ln(1/\delta)$, we obtain the
following lower bound:
\begin{align*}
    \KL(\p\|\phat) & \geq \frac{2\ln(1/\delta)}{3n}\left(\ln\left(\frac{2\ln(1/\delta)(K-1)}{3n \sum_{i = 2}^{K}\phat^0(i)}\right) - 1 \right) + \sum_{i = 2}^{K}\phat^0(i).
\end{align*}
All in all, we have shown that there exists a $\p$, dependent on $\phat$, such
that
\begin{align*}
    \Pp\Big(\KL(\p\|\phat) & \geq \frac{2\ln(1/\delta)}{3n}\left(\ln\left(\frac{2\ln(1/\delta)(K-1)}{3n \sum_{i = 2}^{K}\phat^0(i)}\right) - 1 \right) + \sum_{i = 2}^{K}\phat^0(i)\Big) >\delta\,.
\end{align*}
\end{proof}

\thmLogKLowerBound*
\begin{proof}
  Consider the distributions corresponding to \eqref{eqn:hard_dists}. For each
  of them the probability of the all ones sequence is
  \[
    P_j^n(X_1 = \cdots = X_n = 1) = (1-\tfrac{\alpha}{n})^n > \delta,
  \]
  where the inequality holds because $\ln(1+z) \geq z - z^2$ for $z \geq
  -1/2$, so that
  \[
     \log\Big(1-\frac{\alpha}{n}\Big) 
     \geq -\frac{\alpha}{n} - \Big(\frac{\alpha}{n}\Big)^2
     = -\frac{\alpha}{n} \Big(1 + \frac{\alpha}{n}\Big)
     > -\frac{3\alpha}{2 n} = \frac{\log \delta}{n},
  \]
  where we have used that $n > 2\alpha$ by assumption.

  For an arbitrary weak test $\Psi$, let $j^*$ be the output of $\Psi$
  on the all ones sequence. Then condition~\eqref{eqn:weak_test_fails}
  of Lemma~\ref{lem:weak_testing} is fulfilled because
  \[
    P_{j^*}^n(\Psi = j^*) \geq P_{j^*}^n(X_1 = \cdots = X_n = 1) =
    (1-\tfrac{\alpha}{n})^n > \delta.
  \]
  The lemma therefore tells us that
  \[
    r_n^*(\delta)
      \geq \inf_P \max_j \KL(P_j \| P)
      \geq \inf_P \frac{1}{K-1} \sum_{j=2}^K \KL(P_j \| P)
      = \frac{1}{K-1} \sum_{j=2}^K \KL(P_j \| \bar P),
  \]
  where $\bar P = \frac{1}{K-1} \sum_{j=2}^K P_j$. Since, for any $j$,
  \[
    \KL(P_j \| \bar P)
      = \frac{\alpha}{n} \log \frac{\alpha/n}{\alpha / n / (K-1)}
      = \frac{\alpha}{n} \log (K-1)
      = \frac{2 \log (K-1) \log(1/\delta)}{3 n},
  \]
  we conclude that $r_n^*(\delta) \geq \frac{2 \log (K-1)
  \log(1/\delta)}{3 n}$, as required.

\end{proof}
\lemTsybakovReductionForKL*

\begin{proof}
  Let $\hat P$ be any estimator, and let $\Psi^* = \argmin_{j \in [M]}
  \KL(P_j \| \hat P)$ (breaking ties arbitrarily) be the corresponding
  minimum KL divergence hypothesis test. Then, for any $j$, we have
  \[
    \KL(P_{\Psi^*} \| \hat P) \leq \KL(P_j \| \hat P).
  \]
  Hence, on the event that $\Psi^* \neq j$,
  \begin{align*}
    \KL(P_j \| \hat P)
      &\geq \half \KL(P_j \| \hat P) + \half \KL(P_{\Psi^*} \| \hat P)
      \geq \min_P \Big\{\half \KL(P_j \| P)
          + \half \KL(P_{\Psi^*} \| P)\Big\}\\
      &= \half \KL\Big(P_j \| \frac{P_j + P_{\Psi^*}}{2}\Big)
          + \half \KL\Big(P_{\Psi^*} \| \frac{P_j + P_{\Psi^*}}{2}\Big)
      \geq s_n,
  \end{align*}
  where the last inequality uses \eqref{eqn:pairwise_separation}. Thus,
  for all $j$,
  \[
    P_j^n\big(\KL(P_j \| \hat P) \geq s_n\big) \geq P_j^n\big(\Psi^* \neq
      j\Big).
  \]
  Taking the maximum over $j$ we find that
  \[
    \max_j P_j^n\big(\KL(P_j \| \hat P) \geq s_n\big)
      \geq \max_j P_j^n\big(\Psi^* \neq j\Big)
      \geq \inf_\Psi \max_j P_j^n\big(\Psi \neq j\Big)
      > \delta,
  \]
  where the last inequality holds by assumption \eqref{eqn:test_fails}.
  The result follows by taking the infimum over~$\hat P$.
\end{proof}

\lemWeakTesting*

\begin{proof}
  Let $\hat P$ be any estimator, and let $\hat \Psi = \argmax_{j \in [M]}
  \KL(P_j \| \hat P)$ (breaking ties arbitrarily) be the corresponding
  maximum KL divergence weak hypothesis test. Then we have
  \[
    \KL(P_{\hat \Psi} \| \hat P) = \max_j \KL(P_j \| \hat P) \geq s_n,
  \]
  for $s_n$ as defined in \eqref{eqn:weak_test_lower_bound}. Hence, for
  any $j$,
  \[
    P_j^n\Big(\KL(P_j \| \hat P) \geq s_n\Big)
    \geq P_j^n\big(\hat \Psi = j\big).
  \]
  Taking the maximum over $j$, we find that
  \begin{align*}
    \max_j P_j^n\Big(\KL(P_j \| \hat P) \geq s_n \Big)
      &\geq \max_j P_j^n\big(\hat \Psi = j\big)\\
    \inf_{\hat P} \max_j P_j^n\Big(\KL(P_j \| \hat P) \geq s_n \Big)
      &\geq \inf_\Psi \max_j P_j^n\big(\Psi = j\big)
    > \delta,
  \end{align*}
  where the last inequality holds by
  assumption~\eqref{eqn:weak_test_fails}.
\end{proof}

\lemCover*

\begin{proof}
  We will first consider the upper bound on
  $N(\model_0,\epsilon,\textnormal{KL})$ before proving the lower bound
  on $M(\model_0,\epsilon,V)$.

  \paragraph{Upper Bound:}

  Let $\chi^2(p,q) = \sum_{i=1}^K \frac{(p(i) -
  q(i))^2}{q(i)}$ denote the $\chi^2$ divergence. Then
  $\KL(p\|q) \leq \chi^2(p,q)$ for any $p,q$ in $\model$
  \citep{Tsybakov2009}. Consequently, for any $p,q \in \model_0$,
  \[
    \KL(p\|q)
      \leq \chi^2(p,q)
      = \sum_{i=1}^K \frac{(p(i) - q(i))^2}{q(i)}
      = 2K \sum_{i=1}^K (p(i) - q(i))^2
      = 2K \|p - q\|_2^2.
  \]
  We may equate $\model_0$ with $B_2(\alpha,u) = u + \alpha B_2$, where 
  $B_2$ is the $\ell_2$ unit ball in $\reals^K$. Then by convexity
  of $B_2$, any optimal $\ell_2$ $\epsilon$-cover of $B_2(\alpha,u)$
  will consist of points inside $B_2(\alpha,u)$ and therefore induces a
  KL $(2K\epsilon^2)$-cover of $\model_0$. Hence we get the following
  upper bound on the covering entropy:
  \begin{align*}
    N(\model_0,2K\epsilon^2,\textnormal{KL})
      &\leq N(B_2(\alpha,u),\epsilon,\|\cdot\|_2)
      = N(\alpha B_2,\epsilon,\|\cdot\|_2)\\
      &= N(B_2,\frac{\epsilon}{\alpha},\|\cdot\|_2)
      \leq K \log \Big(\frac{2\alpha}{\epsilon} + 1\Big).\\
    N(\model_0,\epsilon^2,\textrm{KL}) &\leq K \log \Big(\frac{\alpha
    2\sqrt{2K}}{\epsilon} + 1\Big).
  \end{align*}

  \paragraph{Lower Bound:}

  Since the packing entropy $M(\model_0,\epsilon,V)$ is always an upper
  bound on the covering entropy $N(\model_0,\epsilon,V)$, it is
  sufficient to find a lower bound on the covering entropy. Let $B_1$ be
  the unit $\ell_1$ ball in $\reals^K$ and note that $V(p,q) =
  \|p-q\|_1$. Then
  \[
    M(\model_0,\epsilon,V)
      \geq N(\model_0,\epsilon,V)
      = N(\alpha B_2,\epsilon,\|\cdot\|_1).
  \]
  In order to construct a cover of $\alpha B_1$ by $m$ $\ell_1$-balls of
  radius $\epsilon$, we need that the total volume $m \Vol(\epsilon
  B_1)$ of the balls in the cover is at least the volume $\Vol(\alpha
  B_2)$ of the set being covered, so that
  \[
    N(\alpha B_2,\epsilon,\|\cdot\|_1)
      \geq \log \frac{\Vol(\alpha B_2)}{\Vol(\epsilon B_1)}
      = \log \frac{\alpha^K \Vol(B_2)}{\epsilon^K \Vol(B_1)}.
  \]
  The volume of the unit $\ell_p$ ball (for $p\geq 1$) in
  dimension $K$ is
  \[
    \Vol(B_p)
      = 2^K
      \frac{\Gamma\Big(1 + \frac{1}{p}\Big)^K}
           {\Gamma(1 + \frac{K}{p})}.
  \]
  Consequently (using that $\Gamma(2) = 1$ and $\Gamma(3/2) =
  \sqrt{\pi}/2$),
  \begin{align*}
    \log \frac{\Vol(B_2)}{\Vol(B_1)}
    &= \log \frac{\Gamma(1 + K)}{\Gamma(1 + \frac{K}{2})}
    - K \log \frac{2}{\sqrt{\pi}}\\
    &\geq \frac{K}{2} \log \frac{K}{2}
    - K \log \frac{2}{\sqrt{\pi}}
    = \frac{K}{2} \log \frac{K \pi}{8}.
  \end{align*}
  Putting all inequalities together, we conclude that
  \[
    M(\model_0,\epsilon,V)
    \geq
    K \log \frac{\alpha}{\epsilon}
    + \frac{K}{2} \log \frac{K \pi}{8},
  \]
  as claimed.
\end{proof}

\ThmParametricLowerBound*

\begin{proof}
  By Lemma~\ref{lem:covering_numbers}, we can apply
  Theorem~\ref{thm:YangBarron} with
  \begin{align*}
    N(\epsilon)
      &= K \log \Big(\frac{\alpha 2\sqrt{2K}}{\epsilon} + 1\Big),
      &
    M(\epsilon)
      &= K \log \frac{\alpha}{\epsilon}
         + \frac{K}{2} \log \frac{K \pi}{8}.
  \end{align*}
  Let $\alpha = C_1/\sqrt{n}$ for $C_1 \in (0,\frac{\sqrt{n}}{2K}]$ to
  be determined.
  Then
  \[
    \epsilon_n^2
      = \frac{N(\epsilon_n)}{n}
      = \frac{K}{n} \log \Big(\frac{C_1 2\sqrt{2K}}{\epsilon_n \sqrt{n}}
      + 1\Big)
  \]
  has solution $\epsilon_n = C_2 \sqrt{\frac{K}{n}}$ for $C_2 > 0$ such that
  \[
    C_2^2 = \log \Big(\frac{C_1 2\sqrt{2}}{C_2} + 1\Big).
  \]
  Now we find $\loweps_n = C_3  \sqrt{K} \alpha$ for some $C_3 > 0$ such
  that
  \begin{align*}
    M(\loweps_n) &\geq 4n\epsilon_n^2 + 2 \log 2\\
    K \log \frac{\alpha}{\loweps_n} + \frac{K}{2} \log \frac{K \pi}{8}
      &\geq 4 C_2^2 K + 2 \log 2\\
      K \log \frac{1}{C_3\sqrt{K}} + \frac{K}{2} \log \frac{K \pi}{8}
      &\geq 4 C_2^2 K + 2 \log 2\\
      \frac{K}{2} \log \frac{\pi}{8}
      &\geq (4 C_2^2 + \log C_3) K + 2 \log 2,
  \end{align*}
  for which it is sufficient if
  \[
    \frac{1}{2} \log \frac{\pi}{32}
      \geq 4 C_2^2 + \log C_3.
  \]
  The constants cannot be optimized in closed form to maximize
  $\loweps_n = C_1 C_3 \sqrt{\frac{K}{n}}$, but a reasonable choice is to take $C_1 =
  \frac{\sqrt{\log 2}}{2\sqrt{2}}$ (which falls in the allowed range by
  assumption on $n$) such that $C_2 = \sqrt{\log 2}$. This leads to $C_3
  = \exp\Big(\frac{1}{2} \log \frac{\pi}{32} - 4 \log 2\Big)$.

  Having satisfied the conditions of Theorem~\ref{thm:YangBarron}, it
  tells us that
  \[
    \inf_{\hat p} \sup_{\p \in \model_0} \Pr_{\p}n\Big(V(\p,\hat p) \geq
    \frac{C_1 C_3}{2}\sqrt{\frac{K}{n}}\Big) \geq \frac{1}{2}.
  \]
  By Pinsker's inequality this implies that
  \[
    \inf_{\hat p} \sup_{\p \in \model_0} \Pr_{\p}\Big(\KL(\p\|\hat p) \geq
    \frac{C_1^2 C_3^2}{8}\frac{K}{n}\Big) \geq \frac{1}{2} > \delta,
  \]
  and therefore $r_n^*(\delta) \geq C\frac{K}{n}$ for $C = \frac{C_1^2
  C_3^2}{8} \approx 4.1 \times 10^{-6}$. 
\end{proof}

\end{document}